\documentclass{article}

\usepackage{arxiv}

\usepackage[utf8]{inputenc} 
\usepackage[T1]{fontenc}    
\usepackage{hyperref}       
\usepackage{url}            
\usepackage{booktabs}       
\usepackage{amsfonts}       
\usepackage{nicefrac}       
\usepackage{microtype}      
\usepackage{graphicx}
\usepackage{natbib}
\usepackage{doi}

\usepackage{amsmath}
\usepackage{amsthm}
\usepackage{amssymb}
\usepackage{mathtools}
\usepackage{times}
\usepackage{multirow}
\usepackage{caption}
\usepackage{subcaption}
\usepackage{tikz}

\newcommand{\lvfe}{\lambda_{\text{vfe}}}
\newcommand{\lvge}{\lambda_{\text{vge}}}
\newcommand{\R}{\mathbb R}
\newcommand{\qkj}{k_j}
\newcommand{\truek}{\tilde{\bfk}}
\newcommand{\truekj}{\tilde{k}_j}

\newcommand{\qlj}{\lambda_j}

\newcommand{\truelj}{\tilde{\lambda}_j}
\newcommand{\qhj}{h_j}
\newcommand{\trueh}{\tilde{\bfh}}
\newcommand{\truehj}{\tilde{h}_j}
\newcommand{\qnot}{{q_0}}
\newcommand{\qnotj}{{q_0^j}}
\newcommand{\bfl}{{\boldsymbol \lambda}}
\newcommand{\bfh}{{\boldsymbol h}}
\newcommand{\bfk}{{\boldsymbol k}}
\newcommand{\bfbeta}{{\boldsymbol \beta}}
\newcommand{\Crest}{C(\lambdarest, \krest, \betarest)}

\newcommand{\gstar}{g}

\newcommand{\truedist}{p_0}
\newcommand{\triplet}{(p,\truedist,\varphi)}

\newcommand{\varparams}{{\bfl, \bfk, \bfbeta}}

\newcommand{\elbo}{\mathrm{ELBO}}

\newcommand{\normev}{\bar{Z}}
\newcommand{\detnormev}{\bar{Z}_K}
\newcommand{\PsiK}{\Psi(q_0)}
\newcommand{\PsiKn}{\Psi_n(q_0)}
\newcommand{\rlct}{\lambda\triplet}
\newcommand{\multiplicity}{m\triplet}
\newcommand{\transprior}{\varphi(\gstar(\xi)) \vert \gstar'(\xi)\vert}
\newcommand{\baseQ}{\mathcal Q_0}
\newcommand{\normvfe}{\bar{F}_{vb}(n)}
\newcommand{\minnormvfe}{\bar{F}^*_{vb}(n)}

\newcommand{\betarest}{\bfbeta_{[-1]}}
\newcommand{\lambdarest}{\bfl_{[-1]}}
\newcommand{\krest}{\bfk_{[-1]}}

\DeclarePairedDelimiterX{\infdivx}[2]{(}{)}{%
  #1\;\delimsize\|\;#2%
}
\newcommand{\kl}{\mathrm{KL}\infdivx}

\def\R{{\ensuremath {\mathbb R}}}

\def\E{{\ensuremath {\mathbb E}}}
\def\C{{\ensuremath {\mathbb C}}}

\newcommand{\abs}[1]{\left|\, #1\,\right|}

\newcommand{\sqbrac}[1]{\left[ #1 \right]}

\theoremstyle{plain}
\newtheorem{theorem}{Theorem}[section]

\newtheorem{lemma}[theorem]{Lemma}

\theoremstyle{definition}

\theoremstyle{remark}

\title{Variational Bayesian Neural Networks via Resolution of Singularities}

\author{ \href{https://orcid.org/0000-0002-6842-2352}{\includegraphics[scale=0.06]{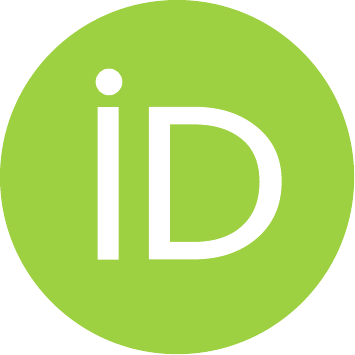}\hspace{1mm} Susan Wei} \\
	School of Mathematics and Statistics\\
	University of Melbourne\\
	Melbourne, Victoria \\
	\texttt{susan.wei@unimelb.edu.au} \\
	\And
	\href{https://orcid.org/0000-0002-7927-1174}{\includegraphics[scale=0.06]{orcid.pdf}\hspace{1mm}Edmund Lau} \\
	School of Mathematics and Statistics\\
	University of Melbourne\\
        Melbourne, Victoria\\
	\texttt{elau1@student.unimelb.edu.au} \\
}

\date{}

\renewcommand{\shorttitle}{VI for BNNs via Resolution of Singularities}

\hypersetup{
pdftitle={Variational Bayesian Neural Networks via Resolution of Singularities},
pdfsubject={},
pdfauthor={Susan Wei, Edmund Lau},
pdfkeywords={First keyword, Second keyword, More},
}

\begin{document}
\maketitle

\begin{abstract}
In this work, we advocate for the importance of singular learning theory (SLT) as it pertains to the theory and practice of variational inference in Bayesian neural networks (BNNs). 
To begin, using SLT, we lay to rest some of the confusion surrounding discrepancies between downstream predictive performance measured via e.g., the test log predictive density, and the variational objective. 
Next, we use the SLT-corrected asymptotic form for singular posterior distributions to inform the design of the variational family itself. 
Specifically, we build upon the idealized variational family introduced in \citet{bhattacharya_evidence_2020} which is theoretically appealing but practically intractable. 
Our proposal takes shape as a normalizing flow where the base distribution is a carefully-initialized generalized gamma. 
We conduct experiments comparing this to the canonical Gaussian base distribution and show improvements in terms of variational free energy and variational generalization error.

\end{abstract}

\section{Introduction}
A Bayesian neural network (BNN) \cite{mackay_probable_1995} is a neural network endowed with a prior distribution $\varphi$ on its weights $w$. Despite their theoretical appeal \cite{lampinen_bayesian_2001, wang_survey_2020}, applying BNNs in practice is not without significant challenges. MCMC and its variants, while widely considered the gold standard, can be prohibitively expensive in terms of computation. On the other hand, fast alternatives such as variational inference may result in \textit{uncontrolled} approximations.

\begin{figure}[t!]
    \centering
    \includegraphics[width=0.4\textwidth]{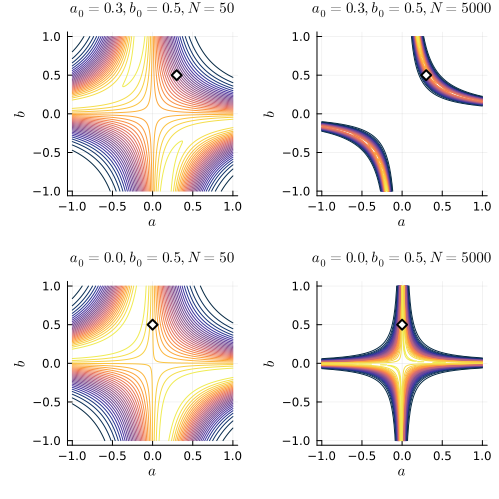}
    \caption
    {
    Posterior density contour plot for a 2D $\tanh$-regression model, $p(y | x, a, b) \propto \exp(y - a \tanh(bx))$. The white diamond marks the true parameter $(a_0, b_0)$ used to generate the dataset $\mathcal D_n$. Each row shows a different true distribution, while each column shows a different sample size $n$. When $a_0 b_0 = 0$ as in the second row, the set of true parameters $ W_0$ is not a singleton and contains a singularity at the origin. It is worth noticing that, for a singular model, even when the truth is not at a singularity (first row), the posterior is still far from being locally Gaussian even at sample size $n=5000$.
    }
    \label{fig:tanh-regression-likelihood-plot}
\end{figure}

In this work, we mine insights from \textbf{singular learning theory} (SLT) \cite{watanabe_algebraic_2009} to explain and improve upon certain aspects of BNNs. Roughly speaking, a model is (strictly) \textbf{singular} if
the parameter-to-model mapping is not one-to-one and the likelihood function does not look Gaussian\footnote{These features should not be viewed as pathological, see ``Deep learning is singular and that's good" by \cite{wei_deep_2022}.}. That neural networks are singular is well documented \cite{sussmann_uniqueness_1992, watanabe_generalization_2000,watanabe_learning_2001,fukumizu_likelihood_2003, watanabe_almost_2007}. We refer the readers to \citet{wei_deep_2022} for a detailed proof in the case of a standard feedforward network. The singular nature of BNNs has interesting implications for the posterior distribution, see Figure \ref{fig:tanh-regression-likelihood-plot}.

Let $(x,y)$ denote the input-target pair modeled jointly as $ p(x,y|w) = p(y \vert x, w)p(x)$ where $w \in \R^d$ is the model parameter. Let $p(y|x,w)$ be a neural network model with functional model $f$, by which we mean $y = f(x,w) + \epsilon$ where $\epsilon$ is some random variable. For example, if we have Gaussian additive noise $\epsilon$, the conditional distribution could be modelled as $\mathcal N(y | f(x,w), \sigma^2 I)$ where $f$ is a feedforward ReLU network with weights $w$.

The central quantity of interest in BNNs is the intractable posterior distribution over the neural network weights,
$$
p(w\vert \mathcal D_n) = \frac{\prod_{i=1}^n p(y_i\vert x_i,w) \varphi(w)}{Z(n)},
$$ 
where $\mathcal D_n = \{(x_i,y_i)\}_{i=1}^n$ is a dataset of $n$ input-output pairs. The normalizing constant,
$$
Z(n)=\int \prod_{i=1}^n p(y_i\vert x_i,w)  \varphi(w) \,d w,
$$
is variously known as the \textbf{model evidence} and the \textbf{marginal likelihood}. Define the \textbf{empirical entropy} of the training data,
$$S_n = - \frac{1}{n} \sum_{i=1}^n \log p_0(y_i \vert x_i).$$
We shall call $$\normev(n) = -\log Z(n) - nS_n$$ the \textbf{normalized evidence}. 
Let us call 
$F(n) := -\log Z(n)$
the \textbf{Bayes free energy} and 
$\bar{F}(n) := -\log \normev(n)$
its normalized version.

Unlike prediction in traditional neural networks, prediction in BNNs proceeds by marginalization, i.e., averaging over all possible values of the network weights. Namely, prediction in BNNs makes use of the \textbf{Bayes posterior predictive distribution},
\begin{equation}
	p(y  \vert x,\mathcal D_n) := \int p(y \vert x, w) p(w \vert \mathcal D_n) \,d w.
	\label{bayes_ppd}
\end{equation}
With \eqref{bayes_ppd}, we can calculate prediction uncertainties as well as obtain better calibrated predictions \cite{heek_well-calibrated_2018, osawa_practical_2019, maddox_simple_2019}. 

In Section \ref{sec:Bayes_posterior_predictive}, we recapitulate from the perspective of SLT the predictive advantages of BNNs over traditional neural networks. Specifically, SLT shows that the Bayes posterior predictive distribution in \eqref{bayes_ppd} has lower generalization error compared to MLE or MAP point estimates.

Despite compelling arguments for employing BNNs, we must reckon with the fact that they can only ever be applied \textit{approximately}. Among approximate techniques, a major class is represented by scaling classic MCMC to modern settings of large datasets and deep neural networks \cite{welling_bayesian_2011,chen_stochastic_2014,zhang_cyclical_2020}. In this paper, we instead turn our focus to variational inference, which is particularly suited to scaling BNNs to large datasets. 

All variational inference techniques are characterized by two ingredients. First, a family of densities $\mathcal Q$, often called the variational family, is posited. Second, some $q^* \in \mathcal Q$ is found via optimization according to some criterion that measures closeness to the desired target density. In this work, we seek to approximate the posterior density and we will employ the conventional Kullback-Leibler divergence. This leads to the optimization problem,
\begin{equation}
	\min_{q \in \mathcal Q} \kl{q(w) }{p(w|\mathcal D_n) }.
	\label{kl_vi_obj}
\end{equation}
This is equivalent to minimizing the so-called \textbf{normalized\footnote{Throughout this paper, we work with normalized quantities for ease of exposition. The asymptotics presented hold equally for the unnormalized counterparts.} variational free energy (VFE)}, 
$$
\normvfe := \E_q  n K_n(w)  + \kl {q(w)}  { \varphi(w)}.  
$$ 
It is easy to see that $\normvfe \ge \bar{F}(n)$ with equality if and only if the variational distribution is exactly equal to the posterior. 
Readers are likely more familiar with the variational objective of maximizing the so-called \textbf{evidence lower bound (ELBO)} which is simply related to the (normalized) VFE via $\operatorname{ELBO} = -\normvfe$.

Let $q^* \in \mathcal Q$ be a minimizer of \eqref{kl_vi_obj}. et us call the variational approximation to \eqref{bayes_ppd} given by 
\begin{equation}
	p_{vb}(y|x,\mathcal D_n) := \int p(y|x,w) q^*(w) \,dw,
	\label{q_posterior_predictive}
\end{equation}
the \textbf{induced predictive distribution}. 
We can measure the predictive accuracy of $p_{vb}$ using once again the KL divergence, i.e., 
$$
G_n(p_{vb}(y|x, \mathcal D_n)):=\kl{p_0(y|x)}{p_{vb}(y|x, \mathcal D_n)},
$$ 
which we shall call the \textbf{variational generalization error (VGE)}. Per the discussion in Section \ref{sec:Bayes_posterior_predictive}, this is, up to a constant and a sign flip, nothing more than the typical \textbf{test log predictive density} \cite{gelman_understanding_2014} commonly employed in variational inference evaluation.

We shall see in Section \ref{sec:two_var_RLCTs} that, surprisingly, the VGE may be arbitrarily high even for a variational family whose minimum VFE is close to optimality. In other words, it is not guaranteed that minimizing \eqref{kl_vi_obj} results in good downstream predictive performance. 
The outlook is not entirely bleak. Depending on the relationship between two critical quantities of variational inference --  the \textbf{MVFE coefficient $\lvfe$} and the \textbf{VGE coefficient $\lvge$} -- the generalization error of the induced predictive distribution may be controllable via minimizing the VFE.

Clarification of the relationship between the two variational coefficients for most common variational learning problems is an open problem, which we leave aside for future work. We will assume the variational coefficients are related \textit{favorably}, in a manner which will be made clear in Section \ref{sec:two_var_RLCTs}, and proceed to design a variational family whose \textbf{variational approximation gap} is small. The proposal is predicated on an important SLT result which states that, roughly speaking, the posterior distribution over the parameters of a singular model is not asymptotically Gaussian, but can still be put into an explicit standard form via the \textbf{resolution of singularities}.


\section{Singular learning theory}
\label{sec:slt}
In this section, we give a succinct overview of key concepts from SLT. We focus in particular on what SLT has to say about the behavior of the posterior distribution in strictly singular models.
Let us assume the parameter space $W$ is a compact set in $\mathbb R^d$ and
$
\truedist(x,y) =  p_0(y \vert x) p(x)
$
is the true data-generating mechanism. Throughout, we suppose there exists $w_0 \in W$ such that $p_0(y \vert x) = p(y \vert x, w_0).$ In the parlance of SLT, this condition is known as \textbf{realizability}. Let $\varphi(w)$ be a compactly-supported prior. We shall refer to 
$
( p(\cdot, \cdot), p_0(\cdot, \cdot), \varphi(\cdot) )
$
as a \textbf{model-truth-prior triplet}. The roles played by compactness and realizability in singular learning theory are discussed in Appendix \ref{appendix:assumptions}.

Define $K(w)$ to be the Kullback-Leibler divergence between the truth and the model, i.e., 
$$
K(w) :=\kl {p_0(x,y) }{ p(x,y\vert w) }.
$$
Following \citet{watanabe_algebraic_2009}, we say a model is \textbf{regular} if 1) it is identifiable, i.e., the map $ w \mapsto p( \cdot, \cdot | w)$ from parameter to model is one-to-one and 2) its Fisher information matrix $I(w)$ is positive definite for arbitrary $w \in W$.
We call a model \textbf{strictly singular} if it is not regular. The term singular will refer to either regular or strictly singular models. 
See Figure \ref{fig:tanh-regression-likelihood-plot} for an example of a strictly singular model with two truth settings. This figure illustrates an important lesson: for strictly singular models, even when the true parameter set $W_0 := \{ w : K(w) = 0\}$ does not contain singularities, the posterior distribution is still far from Gaussian.

The following theorem from \citet{watanabe_algebraic_2009}, adapted for notational consistency, gives precise conditions for the existence of \textbf{resolution maps}, algebraic-geometrical transformations which enables $K(w)$ to be locally written as a \textit{monomial}, i.e., a product of powers of variables such as in the right-hand-side of \eqref{kmono}. The result is itself based on Hironaka's resolution of singularities, a celebrated result in modern algebraic geometry.

To prepare, let $W_\epsilon = \{w \in W: K(w) \le \epsilon\}$ for some small positive constant $\epsilon$ and $W_\epsilon^{(R)}$ be some real open set such that $W_\epsilon \subset W_\epsilon^{(R)}$. The theorem below will make use of the multi-index notation: for a given $\xi = (\xi_1,\ldots,\xi_d) \in \mathbb R^d$, define
$
w^{\bfk} := w_1^{k_1}\cdots w_d^{k_d}
$
where the multi-index $\bfk = (k_1,\ldots,k_d)$ with each $k_j$ a nonnegative integer. Due to space constraints, Fundamental Conditions I and II required below are stated and discussed in Appendix \ref{appendix:assumptions}.

\begin{theorem}[Theorem 6.5 of \cite{watanabe_algebraic_2009}]
	Suppose the model-truth-prior triplet $\triplet$ satisfies Fundamental Conditions I and II with $s=2$. We can find a real analytic manifold $M^{(R)}$ and a proper and real analytic map $g: M^{(R)} \to W_\epsilon^{(R)}$ such that 
	\begin{enumerate}
		\item $M = g^{-1}(W_\epsilon)$ is covered by a finite set $M = \cup_\alpha M_\alpha$ where $M_\alpha = [0,b]^d$.
		\item In each $M_\alpha$, 
		\begin{equation}
			K(g(\xi)) = \xi^{2\bfk} = \xi_1^{2k_1} \cdots \xi_d^{2k_d},
			\label{kmono}
		\end{equation}
		where $k_j \in \mathbb N, j=1, \ldots, d$  are such that not all $k_j$ are zero.
		\item There exists a $C^\infty$ function $b(\xi)$ such that 
		\begin{equation}
			\transprior=  \xi^{\bfh} b(\xi) =  \xi_1^{h_1} \cdots \xi_d^{h_d} b(\xi),
			\label{restmono}
		\end{equation}
		where $h_j \in \mathbb N, j=1,\ldots,d$, $\vert g'(\xi)\vert $ is the absolute value of the determinant of the Jacobian and  $b(\xi)>c>0$ for $\xi \in [0,b]^d$.
	\end{enumerate}
	\label{thm:res}
\end{theorem}

In Theorem \ref{thm:res} we have suppressed the dependency on the manifold chart index $\alpha$, but the reader should keep in mind that the maps $g$ and the multi-indices are all indexed by $\alpha$. 
It is also important to recognize that none of these said quantities are unique for a given triplet $\triplet$.

A crucial quantity that appears in SLT is a rational number in $(0,d/2]$ known as the \textbf{real log canonical threshold} (RLCT). Let $ \{ M_\alpha: \alpha\}$ be as in Theorem \ref{thm:res} and define
$$
\lambda_j = \frac{h_j+1}{2k_j},  j=1,\ldots, d
$$
where $h_j$ and $k_j$ are the entries of the multi-indices $\bfh$ and $\bfk$ in a local coordinate $M_\alpha$. When $k_j =0$, $\lambda_j$ is taken to be infinity. 

Uniquely associated to a triplet $\triplet$ are its real log canonical threshold (RLCT) and its multiplicity defined, respectively, as
\begin{equation}
	\lambda = \min_\alpha \min_{j\in{1,\ldots,d}} \lambda_j, \quad m = \max_\alpha \# \{j:  \lambda_j = \lambda \}.
	\label{rlct_m_def}
\end{equation}
Let $\{\alpha^*\}$ be the set of those local coordinates in which both the $\min$ and $\max$ in \eqref{rlct_m_def} are attained. \citet{watanabe_algebraic_2009} calls this set the \textbf{essential coordinates} and the corresponding collection $\{M_\alpha\}$ the \textbf{essential charts}. 

If $\{w: K(w)=0, \varphi(w)>0\}$ is not the empty set, the RLCT of a model-truth-prior triplet is \textit{at most} $d/2$ \citep[Theorem~7.2]{watanabe_algebraic_2009}. 
When the model is regular, the RLCT is \textit{exactly equal} to $d/2$ and the multiplicity $m=1$ \citep[Remark~1.15]{watanabe_algebraic_2009}. In fact, (twice the) RLCT may be regarded as the effective degrees of freedom in strictly singular models \citep{wei_deep_2022}. The RLCT also shows up in important asymptotic results, see \eqref{Bayes_Gn} and \eqref{normFE_exp}.

Henceforth, to make clear that the RLCT and multiplicity are invariants of the model-truth-prior triplet, we shall write $\rlct$ and $\multiplicity$ to mark this dependence. In Appendix \ref{sec:toy}, we recall a simple toy network, a two-parameter $\tanh$ network, where the resolution map, the RLCT, and the multiplicity can be calculated explicitly.

\subsection{Posterior distribution in singular models}
The posterior distribution in strictly singular models is decidedly not Gaussian. The correct asymptotic form can be derived using SLT.
For a particular manifold chart index $\alpha$, let us apply the transformation $g_\alpha(\xi) = w$ and rewrite the posterior distribution in the new coordinate $\xi$,
\begin{equation}
    p(\xi \vert \mathcal D_n) = \frac{\exp(-nK_n(g_\alpha(\xi)) ) \varphi(g_\alpha(\xi)) \vert g_\alpha'(\xi)\vert}{\normev(n)},
    \label{transformed_posterior}
\end{equation}
with 
$$
K_n(w) = \frac{1}{n} \sum_{i=1}^n \log \frac{p_0(y_i \vert x_i)}{p(y_i\vert x_i, w)}
$$
denoting the sample average log likelihood ratio. Note $K_n(w)$ is the empirical counterpart to $K(w)$.

By (cheekily) substituting \eqref{kmono}  \eqref{restmono} into $p(\xi \vert \mathcal D_n)$ in \eqref{transformed_posterior}, we obtain that the posterior distribution for large $n$, in the chart $M_\alpha$, is described as a so-called \textbf{standard form} \cite{watanabe_mathematical_2018}:
$$
\exp(-n \xi_1^{2k_1} \xi_2^{2k_2} \cdots \xi_d^{2k_d} ) | \xi_1^{h_1} \cdots \xi_d^{h_d} | b(\xi).
$$
In other words, the posterior distribution over the parameters of a singular model can be transformed into a mixture of standard forms, asymptotically. 
In Figure \ref{fig:tanh-regression-likelihood-plot} we display the singular posterior density contour plot for a toy 2D $\tanh$-neural network in two settings of the true distribution. 


\section{The Bayes posterior predictive distribution}
\label{sec:Bayes_posterior_predictive}
Let the generalization error of some predictive distribution $\hat p_n(y|x)$, estimated from a training set $\mathcal D_n$, be measured using the KL divergence:
\begin{equation}
G_n(\hat p_n(y|x)) := \kl {p_0(y|x)p(x)}{\hat p_n(y|x) p(x) }
\label{Gn}
\end{equation}
In the machine learning community, this goes by another name: $G_n(\cdot)$ is, up to a constant and a sign flip, the population counterpart to the commonly reported test log-likelihood, aka the predictive log-likelihood or \textbf{test log-predictive density}. This can be seen by writing
\begin{equation}
\hat G_n = -  \frac{1}{n'} \sum_{(x,y)\in \mathcal D_{n'}} \left( \log p_0(y|x) - \log \hat p_n(y|x) \right)
\label{hatGn}
\end{equation}
where $\mathcal D_{n'}$ is an independent dataset.1
According to Theorems 1.2 and 7.2 in \cite{watanabe_algebraic_2009}, we have, for the Bayes posterior predictive distribution \eqref{bayes_ppd}, 
\begin{equation}
	\E G_n(p(y  \vert x,\mathcal D_n)  ) =\rlct /n + o(1/n)
	\label{Bayes_Gn}
\end{equation}
where the expectation is taken with respect to $\mathcal D_n$. We will call the left hand side \eqref{Bayes_Gn} the expected \textbf{Bayes generalization error}. 
This can be contrasted to the expected generalization error of MLE (and similarly of MAP), which Theorem 6.4 of \cite{watanabe_algebraic_2009} shows to be
$
	\E G_n(p(y  \vert x, \hat w_{mle})  ) = S/n + o(1/n)
$
where $S$, the maximum of a Gaussian process, can be much larger than $\rlct$. 
The situation is markedly different for regular models, where differences between the three estimators become negligible in the large-$n$ regime. 

We briefly outline the derivation of \eqref{Bayes_Gn} as it will inform the narrative on the VGE in the next section. First, for the normalized Bayes free energy, under the Fundamental Conditions I and II discussed in \ref{appendix:assumptions}, it was proven in \citep[Main Theorem 6.2]{watanabe_algebraic_2009} that the following asymptotic expansion holds
\begin{equation}
\bar{F}(n) = \rlct \log n + (m-1) \log \log n + O_P(1). 
\label{normFE_exp}
\end{equation}
The result in \eqref{Bayes_Gn} is then proven using the above expansion together with the well known relationship between the Bayes generalization error and the (normalized) Bayes free energy \citep[Theorem 1.2]{watanabe_algebraic_2009}: 
\begin{equation}
    \E G_n(p(y|x,\mathcal D_n)) = \E \bar{F}(n + 1) - \E \bar{F}(n). 
    \label{diff_free_energy}
\end{equation}
where on the right-hand side, the first expectation is with respect to dataset $\mathcal D_{n+1}$ and the second $\mathcal D_n$.
Due to this relationship, the Bayes free energy shares the \textit{same} coefficient as the Bayes generalization error.

\section{A tale of two variational coefficients}
\label{sec:two_var_RLCTs}
Most applications of variational inference in BNNs labor under the following implicit assumptions: 1) optimizers of the variational objective in \eqref{kl_vi_obj} have good induced predictive distributions, and 2) two variational families can be compared according to the performance of their induced predictive distributions. 
A look at the experimental sections of various works on variational BNNs reveal that these assumptions underlie standard practice \cite{blundell_weight_2015, rezende_variational_2015, louizos_structured_2016, louizos_multiplicative_2017, osawa_practical_2019, swiatkowski_k-tied_2020}. We shall see in this section that these two assumptions do not always hold.
 
Let us associate to a variational family $\mathcal Q$ its \textbf{normalized minimum variational free energy (MVFE)}, $$\minnormvfe:= \min_{q \in \mathcal Q} \normvfe.$$
Asymptotics for the MVFE have so far been addressed on a case-by-case basis for certain models and certain variational families, e.g., Gaussian mean-field variational families for reduced rank regression \cite{nakajima_variational_2007}, nonnegative matrix factorization  \cite{kohjima_phase_2017, hayashi_variational_2020}, normal mixture model \cite{watanabe_stochastic_2006}, hidden Markov model \cite{hosino_stochastic_2005}. 
In all the cited instances above, the asymptotic expansion of the \textbf{average normalized MVFE} takes the form
\begin{align}
	\E \minnormvfe&= \lvfe \log n + o(\log n)
\label{vfe_expansion}
\end{align}
where the expectation is taken over datasets $\mathcal D_n$. 
Note that $\lvfe \ge \rlct$ necessarily \cite{nakajima_variational_2007}. Because the \textbf{variational approximation gap},
\begin{equation}
\mathcal G:=\bar{F}_{vb}^*(n) - \bar{F}(n),
\label{var_approx_gap}
\end{equation} 
is the difference of the (normalized) MVFE and the (normalized) Bayes free energy, the gap boils down to the difference between two coefficients:
$$
\mathcal G \approx \left(\lvfe - \rlct \right) \log n.
$$ 
Now, under some natural conditions\footnote{The predictive distribution should be consistent as $n$ goes to infinity, see the discussion in Chapter 13 of \cite{nakajima_variational_2019}}, the VGE admits the asymptotic expansion,
 \begin{equation}
	\E G_n(p_{vb}(y|x,\mathcal D_n)) = \lvge/n + o(1/n).
	\label{vge_expansion}
 \end{equation}
Importantly, $\lvge \ne \lvfe$ in general, e.g.,  \cite{nakajima_variational_2007}. This is in contrast to the Bayesian posterior predictive distribution in \eqref{bayes_ppd}, where the coefficient of the leading $O(1/n)$ term is precisely the RLCT, $\rlct$. That $\lvge \ne \lvfe$ results from the fact that the relationship \eqref{diff_free_energy} is not valid when a variational approximation to the posterior is employed. 

In Figure \ref{fig:three_scenario}, we illustrate the three possible configurations of the coefficients $\rlct, \lvfe, \lvge$ for a given variational family $\mathcal Q$ and a model-truth-prior triplet. When $\lvfe > \lvge$, we call the setting \textbf{favorable} since minimizing the VFE offers control over the VGE. When $\lvfe<\lvge$, we call this \textbf{unfavorable} since achieving even a small variational approximation gap could result in an induced predictive distribution with high generalization error. The distribution of favorable versus unfavorable settings in practice is unclear, as the exact relationship between $\lvfe$ and $\lvge$ has been derived in a limited number of works. The results in \cite{nakajima_variational_2007} on linear neural networks, aka reduced rank regression, show there are both favorable and unfavorable settings depending on the input and output dimension, the number of hidden units, and a rank measurement on the truth. 

Note that even in favorable settings, we must be careful when comparing two variational families $\mathcal Q_1$ and $\mathcal Q_2$. Figure \ref{fig:compareQs_cartoon} illustrates a scenario where the family $\mathcal Q_1$ incurs a smaller variational approximation gap than $\mathcal Q_2$, but the induced predictive distribution of $\mathcal Q_1$ has $\lvge$ higher than that of $\mathcal Q_2$. This shows that comparing different variational approximations by their test log predictive density is fraught with potential misinterpretations. In order to control the downstream predictive performance, it is thus important to find a variational family with a small approximation gap, so that we can inherit (and sometimes even beat!) the predictive advantages of the exact Bayes posterior predictive distribution \eqref{bayes_ppd}, i.e., achieve $\lvge < \rlct$.

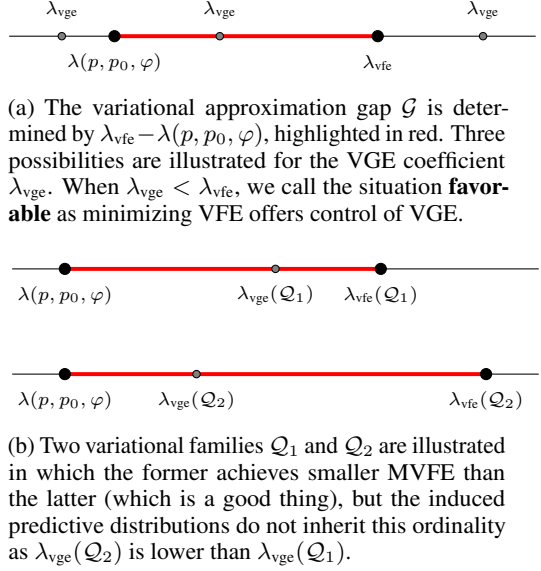
\begin{figure}
\centering
\begin{subfigure}{0.4\textwidth}
    \begin{tikzpicture}[scale=0.7]
    \draw (0, 0) -- (10, 0);
    \draw [line width=1.5pt, red ] (2, 0) -- (7, 0);

    \draw[fill=black] (2, 0) circle (3pt);
    \draw (2, -0.5) node{\scriptsize $\rlct$};
    \draw[fill=black] (7, 0) circle (3pt);
    \draw (7, -0.5) node{\scriptsize $\lvfe$};
    \foreach \i in {1, 4, 9}
    {
        \draw[fill=gray] (\i, 0) circle (2pt);
        \draw (\i, 0.5) node{ \scriptsize $\lvge$};
    }
    \end{tikzpicture}
    \caption{The variational approximation gap $\mathcal G$ is determined by $\lvfe - \rlct$, highlighted in red. Three possibilities are illustrated for the VGE coefficient $\lvge$. When $\lvge < \lvfe$, we call the situation \textbf{favorable} as minimizing VFE offers control of VGE. }
    \label{fig:three_scenario}
\end{subfigure} 
\par\bigskip
\begin{subfigure}{0.4\textwidth}
    \begin{tikzpicture}[scale=0.7]
        \foreach \x/\y/\dx/\i in {9/0/5.5/2, 7/2/2/1}
        {
            \draw (0, \y) -- (10, \y);
            \draw [line width=1.5pt, red ] (1, \y) -- (\x, \y);

            \draw[fill=black] (1, \y) circle (3pt);
            \draw (1, \y -0.5) node{\scriptsize $\rlct$};

            \draw[fill=black] (\x, \y) circle (3pt);
            \draw (\x, \y -0.5) node{\scriptsize $\lvfe(\mathcal{Q}_\i)$};
            \draw[fill=gray] (\x - \dx, \y) circle (2pt);
            \draw (\x - \dx, \y -0.5) node{\scriptsize $\lvge(\mathcal{Q}_\i)$};
        }
    \end{tikzpicture}
    \caption{Two variational families $\mathcal Q_1$ and $\mathcal Q_2$ are illustrated in which the former achieves smaller MVFE than the latter (which is a good thing), but the induced predictive distributions do not inherit this ordinality as $\lvge(\mathcal Q_2)$ is lower than $\lvge(\mathcal Q_1)$.}
    \label{fig:compareQs_cartoon}
\end{subfigure}
\caption{We show in these schematics that evaluating variational approximations to BNNs according to their induced predictive distribution is fraught with potential misinterpretations. }
\label{fig:cartoon}
\end{figure}

\section{Related work}
Although the perspective on offer here -- that the discrepancy between test log predictive density and the variational objective \textit{amounts to the relationship between two variational coefficients} -- is novel, we are not the first to point out this general phenomenon in variational inference \cite{yao_yes_2018, huggins_validated_2020, deshpande_are_2022, dhaka_robust_2020}. 
This phenomenon is also documented in the specific setting of variational inference for BNNs \cite{heek_well-calibrated_2018, yao_quality_2019, krishnan_improving_2020, foong_expressiveness_2020}. For instance, \citet{foong_expressiveness_2020} demonstrated in experiments that optimizing the ELBO may not lead to accurate predictive means or variances. 

Another area of active research in variational BNNs is the design of the variational family itself. For the large part, the mean-field family of fully factorized Gaussian distributions is still predominant in the general practice of variational inference \citep{graves_practical_2011, blundell_weight_2015, hernandez-lobato_black-box_2016,NIPS2016_7750ca35, khan2018fast, sun_functional_2019}. The mean-field assumption is mostly adopted for computational ease, though the limitations are well known \citep{mackay_practical_1992, coker_wide_2022}. Moving beyond mean-field Gaussian, we can find works that make use of more realistic covariance structures \citep{louizos_structured_2016, zhang_noisy_2018} or more expressive approximating families, e.g., via normalizing flows \citep{louizos_multiplicative_2017, papamakarios_normalizing_2021}.  

Finally, we note there have been a few recent works that recognize the non-identifiability of deep learning models \cite{moore_symmetrized_2016, pourzanjani_improving_2017, kurle_detrimental_2022}. These works however seem to treat the non-identifiability as an issue to be fixed.

\section{Methodology}
To achieve a good variational approximation, conventional wisdom says to make $\mathcal Q$ as ``expressive" as possible. We will approach the design of the variational family in a more principled manner using SLT. 
To this end, we rely on recent work in \citet{bhattacharya_evidence_2020} which leveraged SLT to produce an \textit{idealized} variational family as follows. Let $\baseQ$ be a family consisting of generalized gamma distributions in $\mathbb R^d$:
\begin{align}
	&\baseQ = \{\qnot(\xi|\varparams)  = \prod_{j=1}^d \qnotj(\xi_j | \lambda_j, k_j, \beta_j) \} 
	\label{baseQ}
\end{align}
where 
$$
\qnotj(\xi_j  | \lambda_j, k_j, \beta_j) \propto \xi_j^{2\qkj \qlj - 1} \exp(-\beta_j \xi_j^{2\qkj}) 1_{[0,1]}(\xi_j)
$$ 
for $\bfl \in \mathbb R_{>0}^d, \bfk \in \mathbb R_{>0}^d, \bfbeta \in (0,\infty)^d$.
\textbf{Henceforth, let $g:=g_\alpha$ where $\alpha$ is such that $M_\alpha$ is an essential chart.} In other words, we are fixing a resolution map $g$, working in a fixed essential chart domain, and a coordinate $\xi$ on that domain that makes $K(g(\xi))$ a monomial as a function from $\R^d \to \R^d$. 
The idealized variational family of \citet{bhattacharya_evidence_2020} is given as the pushforward of base distributions $q_0 \in \baseQ$ by said map $g$:
\begin{equation}
	\mathcal Q = \{ g \sharp \qnot: \qnot \in \baseQ \}.
	\label{idealizedQ}
\end{equation}
We refer to this as an \textit{idealized} variational family for the simple fact that the resolution map $g$, though its existence is guaranteed, is almost never tractable except in the simplest model-truth-prior triplets. Also note that although the family $\baseQ$ is mean-field, \eqref{idealizedQ} is \textit{not}.

To study the variational approximation gap incurred by the idealized family \eqref{idealizedQ}, we will first introduce some definitions to help us rewrite the gap $\mathcal G$ in notation that is consistent with \citet{bhattacharya_evidence_2020}.
Define
\begin{equation}
	\PsiKn = - \E_{\qnot}  n K_n(\gstar(\xi))  - \kl {\qnot(\xi) }{ \transprior }
	\label{PsiKn}
\end{equation}
See Appendix \ref{appendix:rewrite_G} for the derivation that the variational approximation gap in \eqref{var_approx_gap} is equivalent to
\begin{equation}
\mathcal G = \log \normev(n) - \sup_{\qnot \in \baseQ} \PsiKn.
\label{gap_G}
\end{equation}
Following \citet{bhattacharya_evidence_2020}, we consider the deterministic approximation gap corresponding to \eqref{gap_G}. This is accomplished by replacing $K_n$ with $K$, leading to 
 \begin{equation}
 	\PsiK := - \E_{\qnot}  n K(\gstar(\xi))  - \kl{\qnot(\xi)}{\transprior}.
 	\label{PsiK}
 \end{equation} 
and 
$$ 	
\detnormev(n) := \int_W e^{ - n K(w) } \varphi(w)  \,dw.
$$
For our theoretical investigation, we shall concern ourselves with the \textit{deterministic} variational approximation gap, 
\begin{equation}
\mathcal G_K := \log \detnormev(n) - \sup_{\qnot \in \baseQ} \PsiK.
\label{gap_GK}
\end{equation}
Techniques for generalizing the main result Theorem \ref{thm:deterministic_gap} which concerns $\mathcal G_K$ to the stochastic world can be found in \citet[Section 5.3.3]{plummer_statistical_2021}.

We will appeal to large-$n$ asymptotics to study the behavior of \eqref{gap_GK}. Note that the study and deployment of BNNs is no stranger to large-$n$ asymptotics, both in early \cite{mackay_practical_1992} and recent \cite{ritter_scalable_2018} works. We proceed under this tradition, but deviate from the crude (and incorrect) Laplace approximation that is often employed and instead use the correct asymptotics provided by SLT.

\subsection{Model evidence in singular models}
\label{sec:expandZn}
To study the gap in \eqref{gap_GK}, we begin by examining the asymptotic behavior of $\detnormev(n)$.
When the model is regular, we need not bother with SLT and may find to leading order,
$
\detnormev(n) = \varphi(w_0) \sqrt{ \frac{(2\pi)^d}{\det H(w_0)} }  n^{-d/2}
$ 
via the Laplace approximation. This approximation, however, is egregiously inappropriate for strictly singular models, in particular neural networks \cite{wei_deep_2022}. Nonetheless, perhaps due to a sense that no tractable alternatives exist, the Laplace approximation is seeing a resurgence of application in Bayesian deep learning \cite{ritter_scalable_2018, immer_scalable_2021}. 

For strictly singular models, the quantities $Z(n), \normev(n)$ and $\detnormev(n)$ manifest as singular integrals, i.e., integrals of the form
$
\int_W e^{-nf(w)} \varphi(w) \,dw
$
where $W \subset \mathbb R^d$ is a compact semi-analytic subset, and $f$ and $\varphi$ are real analytic functions. The behavior of a singular integral depends critically on the zeros of $f$.
According to Theorem~6.7 in \cite{watanabe_algebraic_2009}, we find to leading order:
\begin{equation}
	\detnormev(n) = C\triplet n^{-\rlct} (\log n)^{\multiplicity - 1},
	\label{leading_detnormev} 
\end{equation}
where $C\triplet$ is a constant independent of $n$ that we shall call the \textbf{leading coefficient} following the terminology of \cite{lin_algebraic_2011}. 
Note that since $\rlct=d/2$ and $\multiplicity = 1$ in regular models, \eqref{leading_detnormev} is a true generalization of the Laplace approximation, holding for both regular and strictly singular models.

\subsection{Bounding $\mathcal G_K$}
We show in Lemma \ref{lemma:PsiK_lb} in Appendix \ref{appendix:lemmas}, that for large $n$, the following bound holds
\begin{align}
\sup_{\qnot \in \baseQ} \PsiK &\ge - \rlct \log n + C
\label{PsiK_lb}
\end{align}
where $C$ is the constant free of $n$ in Lemma \ref{lemma:PsiK_lb}.
This result is in the same spirit as \citep[Theorem 3.1]{bhattacharya_evidence_2020}, except that we have improved on the tightness of their lower bound, which in turn allows us to devise better initialization of the variational parameters. With Lemma \ref{lemma:PsiK_lb}, we are now in a position to characterize the (deterministic) variational approximation gap, $\mathcal G_K$. 

\begin{theorem}[Deterministic variational approximation gap]
Suppose the model-truth-prior triplet $\triplet$ is such that Theorem \ref{thm:res} holds. Let $g = g_\alpha$ where $\alpha$ is such that $M_\alpha$ is an essential chart. On this essential chart, write the local RLCTs $\truelj = \frac{\truehj+1}{2\truekj},  j=1,\ldots, d$ in descending order so that $\tilde \lambda_1$ is the RLCT of the triplet $\triplet$, i.e., $\tilde \lambda_1 = \rlct$. If the multiplicity of the triplet is 1, we have, for $n$ large,
$
    \mathcal G_K \le  \log C\triplet - C + o(1),
$
where the constant $C$ is as given in Lemma \ref{lemma:PsiK_lb}.
\label{thm:deterministic_gap}
\end{theorem}
All that is needed for the proof of Theorem \ref{thm:deterministic_gap} is to put together the lower bound in Lemma  \ref{lemma:PsiK_lb} with the fact that $\detnormev(n)$ admits the asymptotic expansion in \eqref{leading_detnormev}.  
Even when $\multiplicity \ne 1$, there may be finite $n$ situations when the two terms $(\multiplicity-1) \log \log n$ and $ \log C\triplet - C$ are comparable. In such settings, the idealized variational family $\mathcal Q$ in \eqref{idealizedQ} could still perform well.


\subsection{Learning to desingularize}
\label{sec:learng}
In the preceding section, we studied the deterministic variational approximation gap of an idealized variational family. Although Hironaka proved the existence of a resolution map and showed that it can be found by recursive blow up, known algorithms for finding such resolutions, other than a few exceptional cases (such as those for toric resolutions), have complexity that vastly exceed existing computational capabilities. Thus we are precluded from directly applying the idealized variational family.

This leads us to consider \textit{learning} the resolution map $\gstar$ using an invertible architecture $G_\theta$ resulting in the variational family
\begin{align}
\mathcal{\hat Q} = \{ G_\theta \sharp \qnot(\varparams): \bfbeta=(n,\beta_2,\ldots, \beta_d)\}.
\label{practicalQ}
\end{align}
If the network is expressive enough, we can hope that $g \in \{G_\theta: \theta\}$, which would lead $\mathcal{\hat Q}$ to enjoy the theoretical guarantee provided in Theorem \ref{thm:deterministic_gap}.
Note in \eqref{practicalQ} the first coordinate of $\bfbeta$ has been set to the sample size $n$. The proof of Lemma \ref{lemma:PsiK_lb} reveals why we do so. Specifically, it is shown that the following parameters in $\qnot$ can achieve $\Psi(\qnot) = -\rlct \log n + C$: 
$$
\lambda_1 = \rlct,  \quad k_1 = \tilde k_1, \quad \beta_1 =n
$$ 
where $\tilde k_1$ is as in Theorem \ref{thm:deterministic_gap}. Note that $\rlct$ and $\tilde k_1$ are unknown, but $n$ is certainly known.

It might be readily apparent at this point that we have in $\mathcal{\hat Q}$ a standard normalizing flow, albeit with the base distribution given by the generalized gamma distribution. 
To ease the computational cost, we fix the variational parameters $\bfl, \bfk, \betarest$ and absorb the learning of their optimal values into the invertible transformation $G_\theta$. Note that this is in line with standard practice, whereby normalizing flows adopt parameter-less base distributions. 

To summarize, recognizing that the variational approximation gap can be theoretically studied using SLT allowed for the design of a principled variational family which incurs a variational approximation gap that is independent of sample size $n$, to leading order. To the best of our knowledge, no existing works on normalizing flows for BNNs theoretically address the variational approximation gap. Furthermore, our results offer a new perspective on the benefits of using normalizing flows for variational inference in BNNs.

\section{Experiments}
\label{sec:experiments}

In the following set of experiments\footnote{The code to reproduce our results is available at \href{https://github.com/suswei/BNN_via_SLT}{https://github.com/suswei/BNN\_via\_SLT}.}, we will isolate and examine the effect of the base distribution. 
Specifically, we compare the \textit{generalized gamma base distribution} to the commonly-adopted \textit{Gaussian base distribution}, holding the architecture of $G_\theta$ fixed when we do so. At the outset, we expect that when $G_\theta$ is expressive enough, the effect of the base distribution will be small. However, when $G_\theta$ is more limited (and thus less computationally expensive), we conjecture the generalized gamma base distribution can ``pick up the slack" and outperform the Gaussian base distribution. 

\begin{table}[h]
    \caption{The various model-truth-prior triplets considered in experiments. The truth is realizable. The prior over network weights is standard Gaussian. The RLCT is only known in some of the cases.}
    \vskip 0.1in
    \label{table:triplet_summary}
    \centering
    \scalebox{0.6}{
    \begin{tabular}{llrlrr}
\toprule
model & $H$  &  $\mathrm{dim}_w$ & $\rlct$ &  $\mathrm{dim}_x$ &  $\mathrm{dim}_y$ \\
\midrule
ffrelu & 3   &       42 &         - &      13 &       1 \\
                 & 7   &       98 &         - &      13 &       1 \\
                 & 16  &      224 &         - &      13 &       1 \\
                 & 40  &      560 &         - &      13 &       1 \\
reducedrank & 2   &       14 &       5.0 &       5 &       2 \\
                 & 7   &      119 &      35.0 &      10 &       7 \\
                 & 10  &      230 &      65.0 &      13 &      10 \\
                 & 16  &      560 &     152.0 &      19 &      16 \\
tanh & 15  &       30 &         - &       1 &       1 \\
                 & 50  &      100 &         - &       1 &       1 \\
                 & 115 &      230 &         - &       1 &       1 \\
                 & 280 &      560 &         - &       1 &       1 \\
tanh (zero mean) & 15  &       30 &      1.93 &       1 &       1 \\
                 & 50  &      100 &      3.53 &       1 &       1 \\
                 & 115 &      230 &      5.36 &       1 &       1 \\
                 & 280 &      560 &      8.36 &       1 &       1 \\
\bottomrule
\end{tabular}
    }
    \vskip -0in
\end{table}

In line with our earlier discussion, the parameters of the base distributions are frozen throughout training, see Appendix \ref{appendix:experiment} for the initialization used.
The invertible network $G_\theta$ is implemented as a sequence of affine coupling transformations.
We denote by \texttt{base\_numcouplingpairs\_numhidden} the variational family that results from pushing forward the base distribution through $G_\theta$ with the said configuration, see Appendix \ref{appendix:experiment} for a complete description of the implementation. We consider a total of four different expressivity levels of $G_\theta$ from least to most: \texttt{2\_4}, \texttt{2\_16}, \texttt{4\_4}, \texttt{4\_16}.

The expression for the ELBO objective corresponding to each of the base distributions is given in \eqref{elbo_gengamma} and \eqref{elbo_gaussian} in Appendix \ref{appendix:experiment}. Details of the training procedure such as epochs, learning rate, and optimizer are also given there.
Let $\hat q^*$ be the variational distribution obtained at the end of training. 
Comparison of the base distributions, and hence the two different normalizing flows, will be made according to normalized MVFE, $\minnormvfe$, and VGE, $G_n(p_{vb}(y|x,\mathcal D_n)$. (For both, the lower, the better.) 
We will also estimate the coefficients $\lvfe$ in \eqref{vfe_expansion} and $\lvge$ in \eqref{vge_expansion}, see Appendix \ref{appendix:experiment}.

We consider four model-truth-prior triplets, summarized in Table \ref{table:triplet_summary}, in which the truth is always realizable. In all four triplets, the prior over the neural network weights is chosen to be the standard Gaussian following conventional practice in BNNs \cite{neal_bayesian_1996,bishop_pattern_2006}.  Note, priors for BNNs are notoriously difficult to design and is an area under active research \cite{sun_functional_2019, nalisnick_predictive_2021}.

\begin{figure}[t]
\centering
\includegraphics[width=0.8\textwidth]{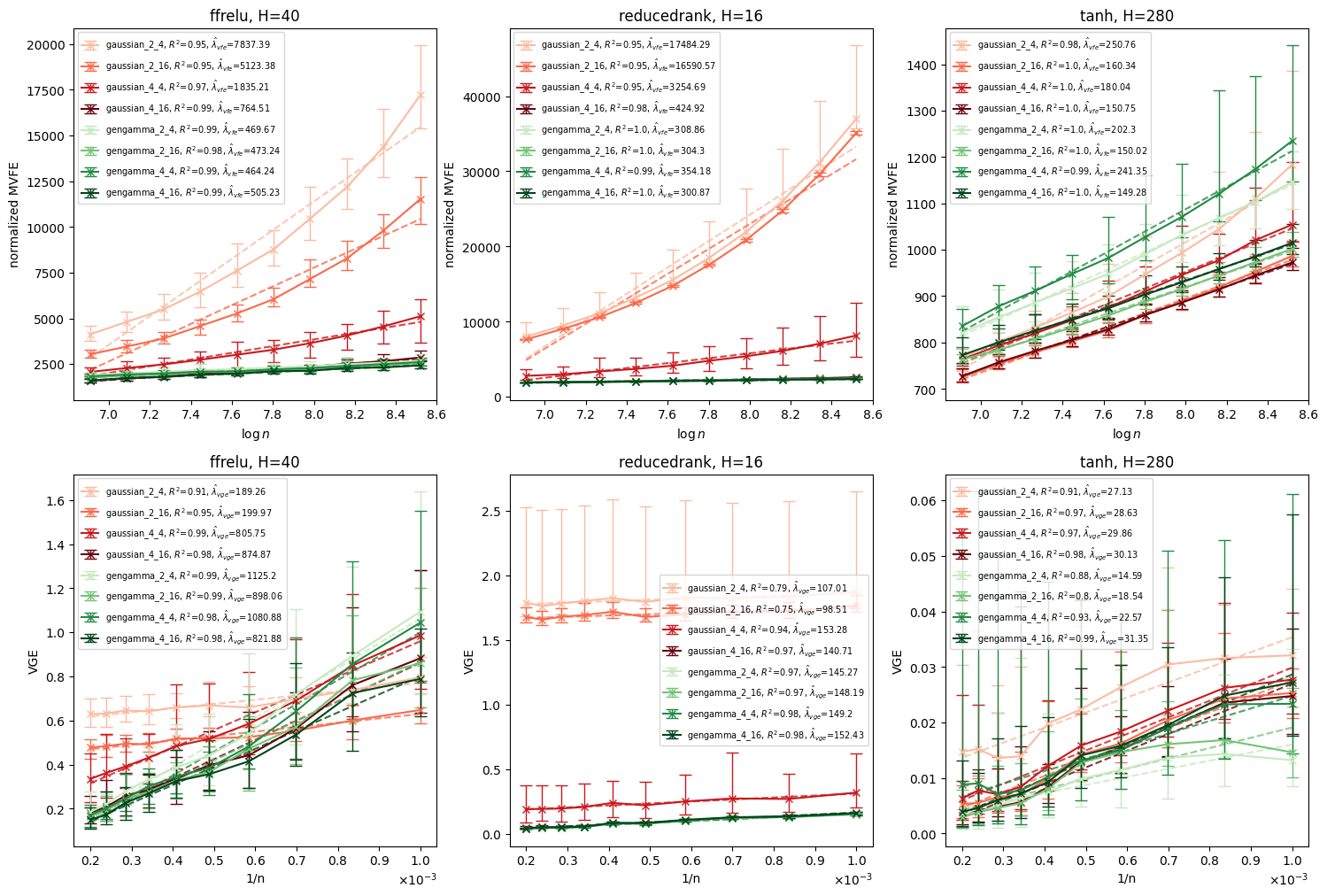}
  \caption{MVFE versus $\log n$ is displayed in the first column and VGE versus $1/n$ is displayed in the second. Each row corresponds to a different model-truth-prior triplet. Line color indicates the expressiveness of the network $G_\theta$, darker being more expressive. Error bars represent mean, min and max over 30 draws of the training set $\mathcal D_n$. The dashed line is the least squares fit with $\lvfe$ and $\lvge$ coefficients and their $R^2$ values displayed in legend.}
  \label{fig:main_text_results}
  \vskip -0.2in
\end{figure}
\subsection{Results}
Due to space constraints, we only show a subset of the results in Figure \ref{fig:main_text_results}; complete results can be found in Appendix \ref{appendix:experiment}. In the first column of Figure \ref{fig:main_text_results}, we plot $\log n$ versus the normalized MVFE. 
First, we observe that when $G_\theta$ is not very expressive, the generalized gamma resoundingly outperforms the Gaussian base distribution for the reduced rank and ReLU experiments across all values of $H$ in terms of achieving lower MVFE. (This can be better seen in Figure \ref{MVFE_vs_logn_gengamma_gaussian_2_4_only} in Appendix \ref{appendix:experiment}.) On the other hand, as conjectured, when $G_\theta$ is most expressive at the \texttt{4\_16} configuration, the distinction in MVFE between the base distributions is still discernible but less dramatic, see Figure \ref{MVFE_vs_logn_gengamma_gaussian_best_line_only}. Interestingly, for the $\tanh$ triplet, the Gaussian base distribution sometimes achieves lower MVFE depending on the configuration of $G_\theta$.

In the second column of Figure \ref{fig:main_text_results}, we plot $1/n$ versus the VGE. The results empirically verify the issues we highlighted in Section \ref{sec:two_var_RLCTs}. In terms of VGE, the generalized gamma is not uniformly better than the Gaussian base distribution for the ReLU experiment, contrary to what the corresponding MVFE plots suggest. Only for the reduced rank experiment do we see one-to-one correspondence between MVFE and VGE. Note that the VGE fit is particularly poor for the Gaussian \texttt{2\_4} and \texttt{2\_16} configurations because these variational approximations are themselves poor. Next, note the scenario in Figure \ref{fig:compareQs_cartoon} is borne out by some of the $\tanh$ experiments. Take for instance $\tanh$ at $H=115$ for the \texttt{2\_4} configuration. Judging by MVFE alone the generalized gamma base is worse than Gaussian base, but the corresponding VGE curves show the opposite, see (3,3) subplot in Figures \ref{MVFE_vs_logn_gengamma_gaussian_2_4_only} and \ref{VGE_vs_inverse_n_gengamma_gaussian_2_4_only}.

\section{Discussion}
We conclude by discussing some of the limitations of the current work. On the empirical front, the reader may have noticed that our experiments did not involve truly deep BNNs.  Strictly speaking this is not a limitation of the proposed method but rather a limitation of the scalability of normalizing flows for approximating deep BNNs. We expect the proposed methodology to benefit from orthogonal research advances in normalizing flow architectures.

On the theoretical side, it may be of interest to flush out the magnitude of $\log C\triplet - C$ in Theorem \ref{thm:deterministic_gap}. The general expression for $C\triplet$, although known in special cases \citep[Corollary 5.9]{lin_algebraic_2011}, has complex dependency on $K(w)$ and the prior. However, we do expect that the leading coefficient can be bounded with some effort. Relatedly, it is important to recognize that Theorem \ref{thm:deterministic_gap} only concerns the variational approximation gap of the idealized family in \eqref{idealizedQ}. Deriving an analogous result for the Gaussian base distribution would make for interesting future work.

We are optimistic that natural conditions on the model-truth-prior triplet and the variational family should allow for general statements about MVFE asymptotic expansions. Further efforts into studying the asymptotics of the MVFE will also advance knowledge of the relationship between $\lvfe$ and $\lvge$. In its place, our results here show that it is all the more important to pay attention to the variational approximation gap if we wish to have useful downstream predictions.

\section*{Acknowledgements}
We thank Daniel Murfet for helpful discussions. SW was supported by the ARC Discovery Early Career Researcher Award (DE200101253). This material is also based on work that is partially funded by an unrestricted gift from Google.

\newpage
\bibliography{references}
\bibliographystyle{unsrtnat}

\newpage
\appendix
\onecolumn
\section{SLT assumptions}\label{appendix:assumptions}

Conventional learning theory studies parametric statistical models under the assumption that they satisfy certain regularity conditions. Unfortunately, most models employed in modern machine learning lack such regularity and exhibit behavior that is unaccounted for by conventional learning theory. The core observation of singular learning theory is that \textbf{singularities} of unidentifiable models have drastic impact on learning behavior. In \cite{watanabe_algebraic_2009} and \cite{watanabe_mathematical_2018}, Watanabe carried out a rigorous investigation into singular statistical models from the Bayesian perspective, culminating in several cornerstone results including Theorem \ref{thm:res} and those described in Section \ref{sec:Bayes_posterior_predictive}. 

Fundamental Conditions I and II given in Definitions 6.1 and 6.3 of \cite{watanabe_algebraic_2009}, respectively, are a set of blanket conditions that Watanabe uses throughout the development of SLT; some components of these conditions are not actually relevant to the results we cite in this paper. Below, we simply present the parts of Fundamental Conditions I and II that are relevant to the SLT results we care about in this paper. 
\begin{enumerate}
    \item The model has compact parameter space $W \subset \R^d$ defined by real analytic inequalities. 

    \item The parameter space $W$ is equipped with a prior distribution with semi-analytic density, i.e. the prior density can be expressed as $\varphi(w) = \varphi_0(w) \varphi_1(w)$ with $\varphi_0$ a positive smooth function and $\varphi_1$ a non-negative analytic function. 
    
    \item For all $w \in W$, $p(x|w)$ has the same support as the truth $p_0(x)$\footnote{In the main text we work with the ``supervised" setting and model the joint distribution $p(x, y | w) = p(x)p(y | x, w)$. Here, for easier exposition, we limit the discussion to the ``unsupervised" setting $p(x | w)$. } 

    \item The true distribution $p_0(x)$ is realisable by the model $p(x | w)$. In other words, there exist a parameter $w_0 \in W$, such that $p_0(x) = p(x | w_0)$. 
    
    \item The log-likelihood ratio function $f(x, w) := \log \frac{p_0(x)} {p(x | w)}$ can be extended to a complex analytic function $W_\C \ni w \mapsto f(\cdot, w)$, taking value in the $L^s(p_0)$ with $s=2$, i.e., the space of functions that are square integrable with respect to the true measure $p_0$. 
\end{enumerate}


\paragraph{On compactness}
We require the parameter space $W$ to be compact in Assumption 1. This is not required when the set of true parameters $W_0 = \{w : K(w) = 0\}$ is contained within a relatively compact neighborhood as contributions of parameters far from $W_0$ drops of exponentially. Even in the case where $W_0$ is not compact, we could consider compactification of $\overline{\R}^d \simeq \R^d \cup \{ |w| = \infty \}$, but we will need to ensure that $f(x, w)$ extends to an analytic function in the neighborhood of infinity. 
In practical implementation however, it is common to have compact $W$ due to machine implementation constraints.

\paragraph{On realizability}
Assumption 4 above required that the zero set of $K(w)$ be non-empty. Let's discuss how to deal with violations of this assumption. In unrealisable cases, we can still derive many SLT results by replacing $K(w)$ with $K(w) - K(w_0)$ where $w_0$ is any parameter that achieves the minimum of $K$ and replacing $f(x,w)$ in Assumption 5 with  $f(x, w) = \log\frac{p(x | w_0)}{p(x | w)}$. Then we can smoothly proceed with the theory in the usual manner by resolving singularities of $K(w) - K(w_0)$ in a neighbourhood of the optimal parameter set $W_0 = \{w : K(w) - K(w_0) = 0\}$, \textit{if} we make an additional assumption known as the \textbf{renormalisability} condition \citep{watanabe_asymptotic_2010}. Without renormalisability, we can still proceed but with considerably more difficult technical challenges \cite{watanabe_asymptotic_2010, nagayasu_asymptotic_2022}. 


\paragraph{On analyticity and integrability conditions} 
The result in Theorem \ref{thm:res} is obtained through a direct application of Hironaka's resolution of singularities, simultaneously, to $K(w) = \int p_0(x) f(x, w) dx$ and the prior $\varphi(w)$. It only requires that the zero set of $K(w)$ is non-empty and both functions are analytic on an open neighbourhood of the zero set. The requirements can be further relaxed to have $K(w)$ and $\varphi(w)$ being semi-analytic and the resolution theorem applied to their analytic factors. 
The analyticity condition on $f(x, w)$ in Assumption 5 is usually sufficient to ensure analyticity on $K(w)$. Application of a resolution map $g(\xi) = w$ for $K(w)$, together with integrability conditions for $f(x, w)$ (Assumption 5) results in the discovery of the connection between geometry $W_0$ with free energy asymptotics \ref{normFE_exp} and \ref{Bayes_Gn} via the RLCT. 



It should be noted, however, that even in cases where $f(x, w)$ is non-analytic, the model might still be ameanable to the same treatment if an equivalent analytic representation can be found. For instance, \citep[Section 7.8]{watanabe_algebraic_2009} shows how the non-analytic $f(x, w)$ for normal mixture models can be analysed in SLT.


\section{Toy example of RLCT calculation}
\label{sec:toy}
We recall Example 27 from \cite{watanabe_mathematical_2018} to illustrate the concepts of resolution map, RLCT and multiplicity for a simple model-truth-prior triplet. For univariate input $x \in [0,1]$ and univariate output $ y \in \mathbb R$, consider the model with parameter $w = (a, b) \in [0,1]^2$ given by
\begin{equation}
p(x,y|w) = \frac{1}{\sqrt{2 \pi}} \exp( -\frac{1}{2} ( y - a \tanh( b x))^2 ).
\label{2dtanh}
\end{equation}
Suppose the prior is uniform, i.e., $\varphi(w) = 1$ and the truth is given by $p_0(x,y) = p(x,y | 0, 0 )$. Then we can easily see that 
$$
K(w) =b^2 a^2 \frac{1}{2} K_0(w), 
$$
where
$$
K_0( w) = \int_0^1 \left( \frac{  \tanh(b x)}{ b } \right)^2 \,dx.
$$
The following desingularization map puts the triplet in standard form:
\begin{align*}
\xi_1 &= \sqrt{	\frac{K_0(w)}{2}	} a \\
\xi_2 &= b.
\end{align*}
Next, we have $\varphi(g(\xi)) = \xi^\bfh$ where $h=(0,0$) and $b(\xi) = |g'(\xi)|$. 
Since $(k_1,k_2) = (1,1)$ and $(h_1,h_2) = (0,0)$ we have $(\lambda_1, \lambda_2) = (1/2, 1/2)$. Therefore for this particular model-truth-prior triplet, the RLCT is $1/2$ with multiplicity 2. 

We should note that, to date, there is a rather small collection of strictly singular model-truth-prior triplets where the RLCT and multiplicity are known.

\section{Rewriting the variational approximation gap}
\label{appendix:rewrite_G}
Recall the posterior distribution in the new coordinate $\xi$ in \eqref{transformed_posterior}. For $q$ in \eqref{idealizedQ}, we have
\begin{align*}
	&\kl {q(w)}{ p(w|\mathcal D_n) } \\
    &= \kl {\qnot(\xi) }{  p(\xi \vert \mathcal D_n) } \\
	&= \E_{\qnot}  n K_n(\gstar(\xi)) + \kl {\qnot(\xi)}{\transprior} + \log \normev(n). 
\end{align*} 
Following the notation in \citet{bhattacharya_evidence_2020}, we defined
\begin{equation*}
	\PsiKn = - \E_{\qnot}  n K_n(\gstar(\xi))  - \kl {\qnot(\xi) }{ \transprior }.
\end{equation*}
As long as the support of $\qnot(\xi)$ is contained in the support of the posterior $p(\xi \vert \mathcal D_n)$, we have $\kl {\qnot(\xi) }{  p(\xi \vert \mathcal D_n)} \ge 0 $, leading to the lower bound
$$\PsiKn \le  \log \normev(n).$$	
Equality is achieved if and only if $\qnot(\xi) = p(\xi\vert \mathcal D_n)$.

\section{Lemmas and proofs}
\label{appendix:lemmas}

\begin{lemma}
Suppose the model-truth-prior triplet $\triplet$ is such that Theorem \ref{thm:res} holds. Let $g = g_\alpha$ where $\alpha$ is such that $M_\alpha$ is an essential chart. On this essential chart, write the local RLCTs 
$$
\truelj=\frac{\truehj+1}{2\truekj},  j=1,\ldots, d
$$
in descending order so that $\tilde \lambda_1$ is the RLCT of the triplet $\triplet$, i.e., $\tilde \lambda_1 = \rlct$.
Let $\baseQ$ be as in \eqref{baseQ}. For $q_0 \in \baseQ$, we have
$$
\PsiK = -E_1 -E_2 + E_3 + E_4
$$
with the individual terms $E_1, \ldots, E_4$ given below in the body of the proof. 
\label{lemma:PsiKterms}
\end{lemma}

\begin{proof}
	With standard form and Main Formula 1 in \cite{watanabe_algebraic_2009}, we have
	\begin{align*}
		n K(\gstar(\xi)) &= n \xi^{2\truek}\\
		\transprior &= b(\xi) \abs{\xi^{\trueh}}
	\end{align*}
	with $b(\xi) > 0$.
	We therefore have
	\begin{align*}
		\PsiK 
		&= - n\E_\qnot \sqbrac{\xi^{2\truek}} - \E_\qnot \log \qnot + \E_\qnot \sqbrac{\log \xi^{\trueh}}  +  \E_\qnot \sqbrac{\log b(\xi)}  \\
		&= - E_1 - E_2 + E_3 + E_4
	\end{align*}
	where we have named each term in the sum 
	\begin{align*}
		E_1 &:= n\E_\qnot\sqbrac{\xi^{2\truek}} , \quad E_2 := \E_\qnot \log \qnot \\
		E_3 &:= \E_\qnot \sqbrac{\log \xi^{\trueh}}, \quad E_4 := \E_\qnot \sqbrac{\log b(\xi)} 
	\end{align*}

	In the following we will make use of the following elementary facts about the univariate generalized gamma density truncated to $[0,1]$. They are stated in the same notaton as in \cite{bhattacharya_evidence_2020}. The normalizing constant of $q_j$ is given by $B(\lambda_j,k_j, \beta_j)$ where
	\begin{equation}
	B(\lambda, k,\beta) = \frac{\beta^{-\lambda} \Gamma(\lambda) \gamma(\lambda,\beta)}{2k}
	\label{B}
	\end{equation}
	and $\gamma(a,x) = \frac{1}{\Gamma(a)} \int_0^x t^{a-1} e^{-t} \,d t$ is the (regularized) lower incomplete gamma function. The quantity $E_{q_j} \xi^{2k_j} = G(\lambda_j,\beta_j)$ where
	\begin{equation}
		G(\lambda, \beta) = \frac{\lambda}{\beta} \frac{\gamma(\lambda+1,\beta)}{\gamma(\lambda,\beta)}.
		\label{G}
	\end{equation}

	First we have
	\begin{align*}
		E_1 &= n \prod_{j=1}^d \E_{\qnotj} \xi^{2\truekj} \\
		&= n \prod_{j=1}^d \beta_j^{-\frac{\truekj}{\qkj}} \frac{\Gamma( \qlj + \frac{\truekj}{\qkj}) \gamma(\qlj + \frac{\truekj}{\qkj}, \beta_j )}{ \Gamma(\qlj ) \gamma(\qlj , \beta_j )}.
	\end{align*}
	Next we have
	\begin{align*}
		E_2 &= \sum_{j=1}^d \qhj \E_\qnotj \log \xi_j - \beta_j G(\qlj, \beta_j) - \log B(\qkj, \qhj, \beta_j).
	\end{align*}
	For the third term we have
	\begin{align*}
		E_3 &= \sum_{j=1}^d \truehj \E_\qnotj \log \xi_j.
	\end{align*}
\end{proof}

In the lemma below, we improve upon the lower bound provided in Theorem  3.1 in \cite{bhattacharya_evidence_2020} where the constant is given by
$$
\tilde \lambda  (1 - \prod_{j=m+1}^d G(\tilde \lambda_j, \beta_j)) + \sum_{j=m+1}^d [ \beta_j G(\tilde \lambda_j,\beta_j) + \log B(\tilde k_j, \tilde h_j, \beta_j) ]- \sum_{j=1}^d \log(2 \tilde k_j) - \sum_{j=1}^d \log(\tilde \lambda_j),
$$
where $\beta_j = 1$ for $j \ge m+1$. 

\begin{lemma}
Suppose the conditions of Lemma \ref{lemma:PsiKterms} hold. We have, for $n$ large,
$$
\sup_{\qnot \in \baseQ} \Psi(\qnot) \ge -\rlct \log n + C,
$$
where 
$$
C = \sup_{\lambdarest, \krest, \betarest} \Crest.
$$
\label{lemma:PsiK_lb}
\end{lemma}

\begin{proof}
Let $\qnot$ be such that 
$$
\lambda_1 = \tilde \lambda_1, \quad k_1 = \tilde k_1, \quad \beta_1 =n.
$$
We can use Lemma \ref{lemma:PsiKterms} to obtain the expression for $\PsiK$. Next, using the fact that $n G(\lambda,n) \approx \lambda$ and $\log B(k,h,n) \approx - \lambda \log n$ and $b(\xi)$ is bounded below away from zero, $b(\xi) > b_0 := \inf_\xi b(\xi) > 0$. We get that for $n$ large, 
$$
\sup_\qnot \PsiK \ge -\rlct \log n + \Crest,
$$
where 
\begin{align}
	\Crest &= \rlct \left( 1 -  \prod_{j=2}^d \beta_j^{-\frac{\truekj}{\qkj}} \frac{\Gamma( \qlj + \frac{\truekj}{\qkj}) \gamma(\qlj + \frac{\truekj}{\qkj}, \beta_j )}{ \Gamma(\qlj ) \gamma(\qlj , \beta_j )} \right) \nonumber \\
	&+ \sum_{j=2}^d \left ( (\truehj - \qhj) \E_\qnotj \log \xi_j + \beta_j G(\qlj, \beta_j) - \log B(\qkj, \qhj, \beta_j) \right) + \log b_0.
	\label{C_rest}
\end{align}
\end{proof}

\section{Experiment details}
\label{appendix:experiment}
We first provide details on the model-truth-prior triplets considered in Section \ref{sec:experiments}. Next we describe the architecture adopted for $G_\theta$ in the implementation of the normalizing flow. We then detail the training procedure for learning the normalizing flow and the estimation of the evaluation measures. Finally, additional experimental results are given and discussed.

\subsection{Model-truth-prior triplets}

In all triplets considered, the prior over the neural network weights is chosen to be the standard Gaussian.

In the \textbf{one-layer $\tanh$} experiment, the input $x \in \mathbb R$ follows the uniform distribution on $[-1,1]$, and the response variable $y \in \mathbb R$ is modeled as
$$
p(y\vert x,w) = \frac{1}{\sqrt{2\pi}} \exp( - \frac{1}{2} (y-f(x,w))^2 ),
$$
where 
$$f(x,w) = \sum_{h=1}^H b_h \tanh(a_h x)$$ 
is a $\tanh$ network with $H$ hidden units and $w$ is the collection of neural network weights $\{(a_h,b_h)\}_{h=1}^H$. 
We shall consider two true distributions, one in which we know the true RLCT and multiplicity, which we call \textbf{one-layer $\tanh$ zero-mean}, and the other where we do not, which we call simply \textbf{one-layer $\tanh$}. For the \textit{zero-mean} setting, we set 
$$  
p_0(y\vert x) = p(y\vert x,0) =\frac{1}{\sqrt{2\pi}} \exp( - \frac{1}{2} y^2 ).
$$ 
In this case, it was shown in \cite{aoyagi_resolution_2006} that
$$
\rlct = \frac{H+i^2 + i}{4i+2}
$$
and $ m = 2$ if $i^2 = H$, and $m=1$ if $i^2 < H$ where $i$ is the maximum integer satisfying $i^2 \le H$.  In contrast, were this a regular statistical model,  we would have $\rlct = H$.  
For the other truth setting, we simply take a fixed draw of $w_0$ from the standard Gaussian. In this case the true RLCT and multiplicity are \textit{unknown}. 

In the \textbf{reduced rank regression} experiment, the input $x \in \mathbb R^M$ is generated from standard Gaussian and the response variable $y \in \mathbb R^N$ is modeled as 
$$
p(y\vert x,w) = (2\pi)^{-N/2} \exp \{ -\frac{1}{2} \vert \vert  y-BAx\vert \vert ^2 \},
$$
where $\{w = (A,B) \vert  A \in \mathbb R^{H \times M}, B \in \mathbb R^{N \times H}\}$. 
This model is readily seen to be a special case of a neural network with hidden units $H$ and identity activation function. We shall set $M = H + 3$ and $N = H$. 
The true parameters $A_0$ and $B_0$ are given as follows. 
The matrix $B_0$ is set to be the identity matrix $I_{N \times N}$. The matrix $A_0$ is set to be an identity matrix with dimension $H$ plus three additional columns of 1: $A_0 = [	I_{H \times H}; J_{H \times 3}]$. The rank $r$ for $B_0A_0$ equals $H$.
Under this condition, $N + H < M +r$ is trivially satisfied and we are in Case iii) of \citet{aoyagi_stochastic_2005} for which the RLCT was derived in \cite{aoyagi_stochastic_2005} to be
$$
\rlct = (NH - Hr + Mr)/2,  m=1.
$$
Note that were this a regular model, we would instead have $\rlct=(MH+NH)/2$. Notably the multiplicity is always either $m=1$ or $m=2$ for the reduced rank regression model. 

In the \textbf{feedforward ReLU} experiment, the input $x \in \mathbb R^{13}$ is generated from the standard multivariate Gaussian and the response variable $y \in \mathbb R$ is modeled as Gaussian $N(f(x,w),1)$ where $f(x, w) = w_2 \operatorname{ReLU}(w_1 x)$ for $w_1 \in \mathbb R^{H \times 13}$ and $w_2 \in \mathbb R^{1 \times H}$. The true distribution $p_0(y|x)$ is fixed at a random draw of $w_1,w_2$ from the standard Gaussian. The true RLCT and multiplicity are \textit{unknown} for this truth-prior-triplet.

\subsection{Normalizing flow}

The generalized gamma base distribution $\qnot$ is initialized (and frozen) at
\begin{align*}
	\bfl_0 &= ( 1, \ldots, 1), \\
	\bfk_0 &= (1, \ldots, 1), \\
	\bfbeta_0 &= (n, d/2, \ldots, d/2).
\end{align*}
The Gaussian base distribution is initialized (and frozen) at the standard multivariate Gaussian with mean zero and identity covariance. Only the weights $\theta$ in the invertible architecture $G_\theta$ are updated. 

Next. we detail the implementation of $G_\theta$. With $r$ denoting a binary mask, a so-called affine coupling layer acts as follows for $u, v \in \mathbb R^d,$
$$
u \mapsto v = (1-r) \odot u + r \odot (u \odot \exp(s( r \odot u)) + t(r \odot u) ),
$$
where $s$ and $t$ are scaling and translation networks, respectively. We implement the translation network $t$ as a two-hidden-layer feedforward (leaky) ReLU neural network with $\tanh$ output activation function. The scaling $t$ is another two-hidden-layer feedforward (leaky) ReLU neural network with identity output activation function. 
Note the binary mask $r$ must alternate from one affine coupling layer to the next, for otherwise there would be little expressive power in the resulting network. Note that the specific architecture of $G_\theta$ has rendered the log Jacobian term, $\log \vert G_\theta'(\cdot) \vert$, computationally tractable. 
Below is a printout of the network $G_\theta$ with 2 alternating coupling pairs and 4 hidden units:

\begin{tiny}
\begin{verbatim}
  (s): ModuleList(
    (0): Sequential(
      (0): Linear(in_features=210, out_features=4, bias=True)
      (1): LeakyReLU(negative_slope=0.01)
      (2): Linear(in_features=4, out_features=4, bias=True)
      (3): LeakyReLU(negative_slope=0.01)
      (4): Linear(in_features=4, out_features=4, bias=True)
      (5): LeakyReLU(negative_slope=0.01)
      (6): Linear(in_features=4, out_features=210, bias=True)
      (7): Tanh()
    )
    (1): Sequential(
      (0): Linear(in_features=210, out_features=4, bias=True)
      (1): LeakyReLU(negative_slope=0.01)
      (2): Linear(in_features=4, out_features=4, bias=True)
      (3): LeakyReLU(negative_slope=0.01)
      (4): Linear(in_features=4, out_features=4, bias=True)
      (5): LeakyReLU(negative_slope=0.01)
      (6): Linear(in_features=4, out_features=210, bias=True)
      (7): Tanh()
    )
    (2): Sequential(
      (0): Linear(in_features=210, out_features=4, bias=True)
      (1): LeakyReLU(negative_slope=0.01)
      (2): Linear(in_features=4, out_features=4, bias=True)
      (3): LeakyReLU(negative_slope=0.01)
      (4): Linear(in_features=4, out_features=4, bias=True)
      (5): LeakyReLU(negative_slope=0.01)
      (6): Linear(in_features=4, out_features=210, bias=True)
      (7): Tanh()
    )
    (3): Sequential(
      (0): Linear(in_features=210, out_features=4, bias=True)
      (1): LeakyReLU(negative_slope=0.01)
      (2): Linear(in_features=4, out_features=4, bias=True)
      (3): LeakyReLU(negative_slope=0.01)
      (4): Linear(in_features=4, out_features=4, bias=True)
      (5): LeakyReLU(negative_slope=0.01)
      (6): Linear(in_features=4, out_features=210, bias=True)
      (7): Tanh()
    )
  )
  (t): ModuleList(
    (0): Sequential(
      (0): Linear(in_features=210, out_features=4, bias=True)
      (1): LeakyReLU(negative_slope=0.01)
      (2): Linear(in_features=4, out_features=4, bias=True)
      (3): LeakyReLU(negative_slope=0.01)
      (4): Linear(in_features=4, out_features=4, bias=True)
      (5): LeakyReLU(negative_slope=0.01)
      (6): Linear(in_features=4, out_features=210, bias=True)
    )
    (1): Sequential(
      (0): Linear(in_features=210, out_features=4, bias=True)
      (1): LeakyReLU(negative_slope=0.01)
      (2): Linear(in_features=4, out_features=4, bias=True)
      (3): LeakyReLU(negative_slope=0.01)
      (4): Linear(in_features=4, out_features=4, bias=True)
      (5): LeakyReLU(negative_slope=0.01)
      (6): Linear(in_features=4, out_features=210, bias=True)
    )
    (2): Sequential(
      (0): Linear(in_features=210, out_features=4, bias=True)
      (1): LeakyReLU(negative_slope=0.01)
      (2): Linear(in_features=4, out_features=4, bias=True)
      (3): LeakyReLU(negative_slope=0.01)
      (4): Linear(in_features=4, out_features=4, bias=True)
      (5): LeakyReLU(negative_slope=0.01)
      (6): Linear(in_features=4, out_features=210, bias=True)
    )
    (3): Sequential(
      (0): Linear(in_features=210, out_features=4, bias=True)
      (1): LeakyReLU(negative_slope=0.01)
      (2): Linear(in_features=4, out_features=4, bias=True)
      (3): LeakyReLU(negative_slope=0.01)
      (4): Linear(in_features=4, out_features=4, bias=True)
      (5): LeakyReLU(negative_slope=0.01)
      (6): Linear(in_features=4, out_features=210, bias=True)
    )
  )
)
\end{verbatim}
\end{tiny}

\subsection{Training}
To train the normalizing flow with a generalized gamma base distribution, we first begin by noting that the generalized gamma distribution is simply related to the gamma distribution. Let $V_j$ be a gamma random variable with shape $\qlj$ and rate $\beta_j$, then $V_j^{1/(2\qkj)}$ has density
$
\xi_j^{2\qkj \qlj - 1} \exp(-\beta_j \xi_j^{2\qkj}).
$
This is convenient because the pathwise derivative for the gamma distribution is readily available in libraries such as PyTorch. 
The corresponding optimization objective is given by
\begin{align}
\elbo(\theta) &:= \E_{\xi \sim \qnot} \left [ \sum_{i=1}^n \log p(y_i\vert x_i,G_\theta(\xi)) + \log (\varphi(G_\theta(\xi)) \vert  G_\theta'(\xi) \vert ) \right] - \E_{\xi \sim \qnot}  \log \qnot(\xi)  \nonumber \\
&= \E_{v \sim  Gamma(\bfl_0,\bfbeta_0)} \left [ \sum_{i=1}^n \log p(y_i\vert x_i,G_\theta(v^{1/(2\bfk)})) + \log (\varphi(G_\theta(v^{1/(2\bfk)})) \vert  G_\theta'(v^{1/(2\bfk)}) \vert ) \right] - \E_{\xi \sim \qnot}  \log \qnot(\xi). 
\label{elbo_gengamma}
\end{align}
The number of epochs was set to 5000 for full-batch training using ADAM with initial learning rate of 0.01 for $\theta$ in $G_\theta$. We estimate the expectation in the ELBO using $M=10$ samples except for the entropy component, $-\E_{\qnot} \log \qnot$, which was derived analytically, using Equations \eqref{B} and \eqref{G}.

The exact same training parameters were used for the normalizing flow with Gaussian base distribution in which case the ELBO is given by
\begin{align}
\elbo(\theta):= &\E_{\xi \sim N(0,I_d)} \left [ \sum_{i=1}^n \log p(y_i\vert x_i,G_\theta(\xi)) + \log (\varphi(G_\theta(\xi)) \vert  G_\theta'(\xi) \vert ) \right] + H(N(0,I_d)),
\label{elbo_gaussian}
\end{align}
where $H(\cdot)$ denotes the entropy of the distribution, i.e., $H(\qnot) := -\E_{\xi \sim \qnot}  \log \qnot(\xi)$.

\subsection{Estimating MVFE and VGE}
To estimate normalized MVFE, we use the learned $\hat q^*$ and the empirical training entropy $S_n$ which we know in simulations. To estimate expectations over $\hat q^*$, 1000 samples are used. The VGE is estimated using \eqref{hatGn} with an independent dataset $\mathcal D_{n'}$ of sample size $n'=10000$. 

To estimate coefficients $\lvfe$ and $\lvge$, we generate 30 realizations of training data $\mathcal D_n$ for each of 10 possible sample sizes $n$ evenly spaced on the log scale between 3.0 and 3.7: $n \in$ \{1000, 1196, 1431, 1711, 2047, 2448, 2929, 3503, 4190, 5012\}. This allows us to estimate the left-hand sides of \eqref{vfe_expansion} and  \eqref{vge_expansion}. The coefficients themselves are estimated by fitting least squares, against $\log n$ for the average normalized MVFE and $1/n$ for the average VGE. For the former, we fit an intercept, while for the latter the intercept is forced to be zero.

\subsection{Additional experimental results}
In this section, we display the MVFE and VGE for all four experiments in Table \ref{table:triplet_summary}. We first group by the individual base distributions, which allows for greater readability as the y-axis scale is consistent within the base distribution. We then juxtapose the ``best" performing $G_\theta$, according to MVFE, for each base distribution, which usually happens to be the architecture $G_\theta$ with 4 alternating pairs and 16 hidden units. Similarly, we also plot the least expressive $G_\theta$ which is the \texttt{2\_4} configuration.

\begin{figure}[ht]
\vskip 0.2in
\begin{center}
\centerline{\includegraphics[width=\columnwidth]{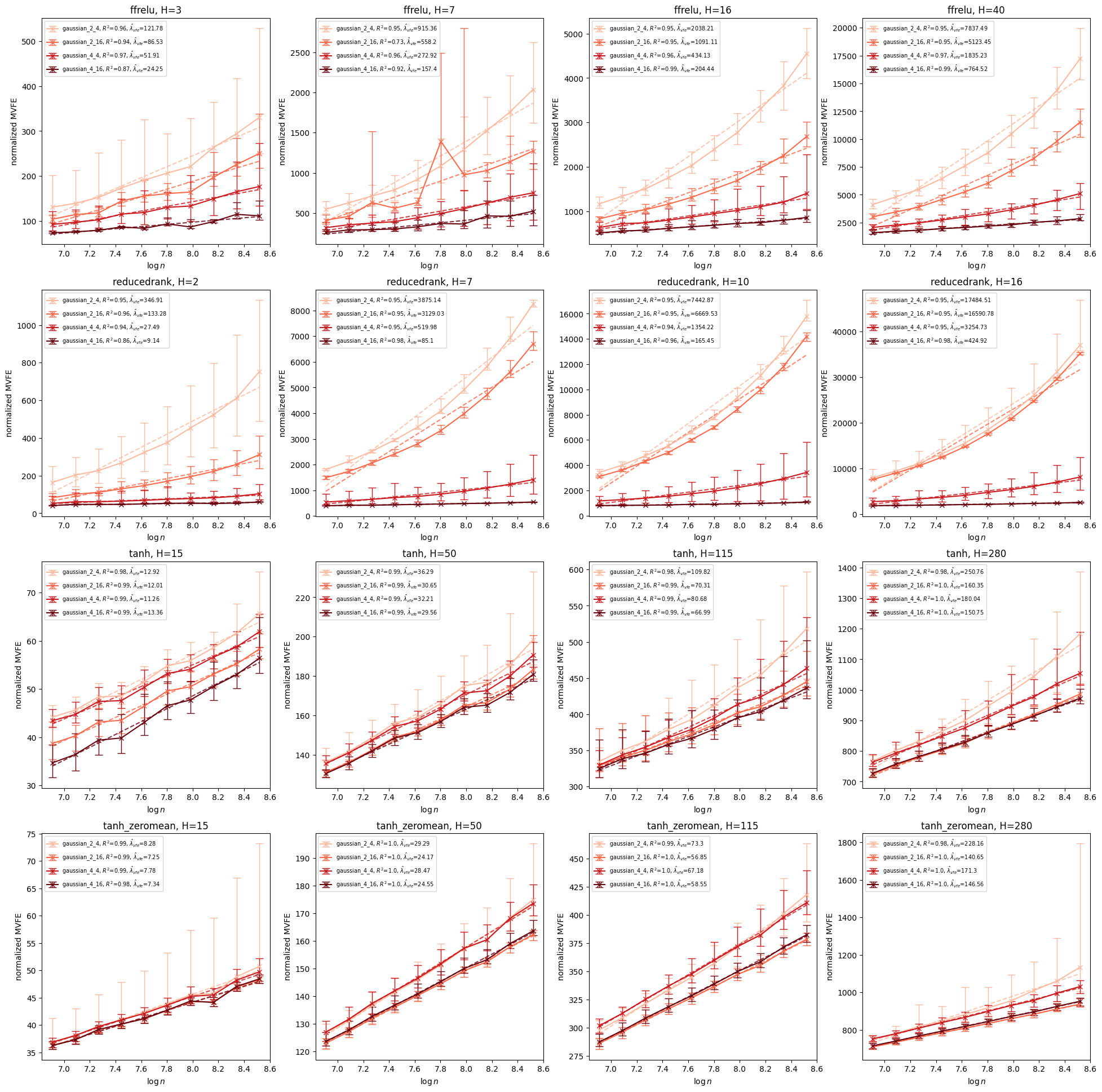}}
\caption{MVFE for Gaussian base distribution.}
\label{MVFE_vs_logn_gaussian}
\end{center}
\vskip -0.2in
\end{figure}

\begin{figure}[ht]
\vskip 0.2in
\begin{center}
\centerline{\includegraphics[width=\columnwidth]{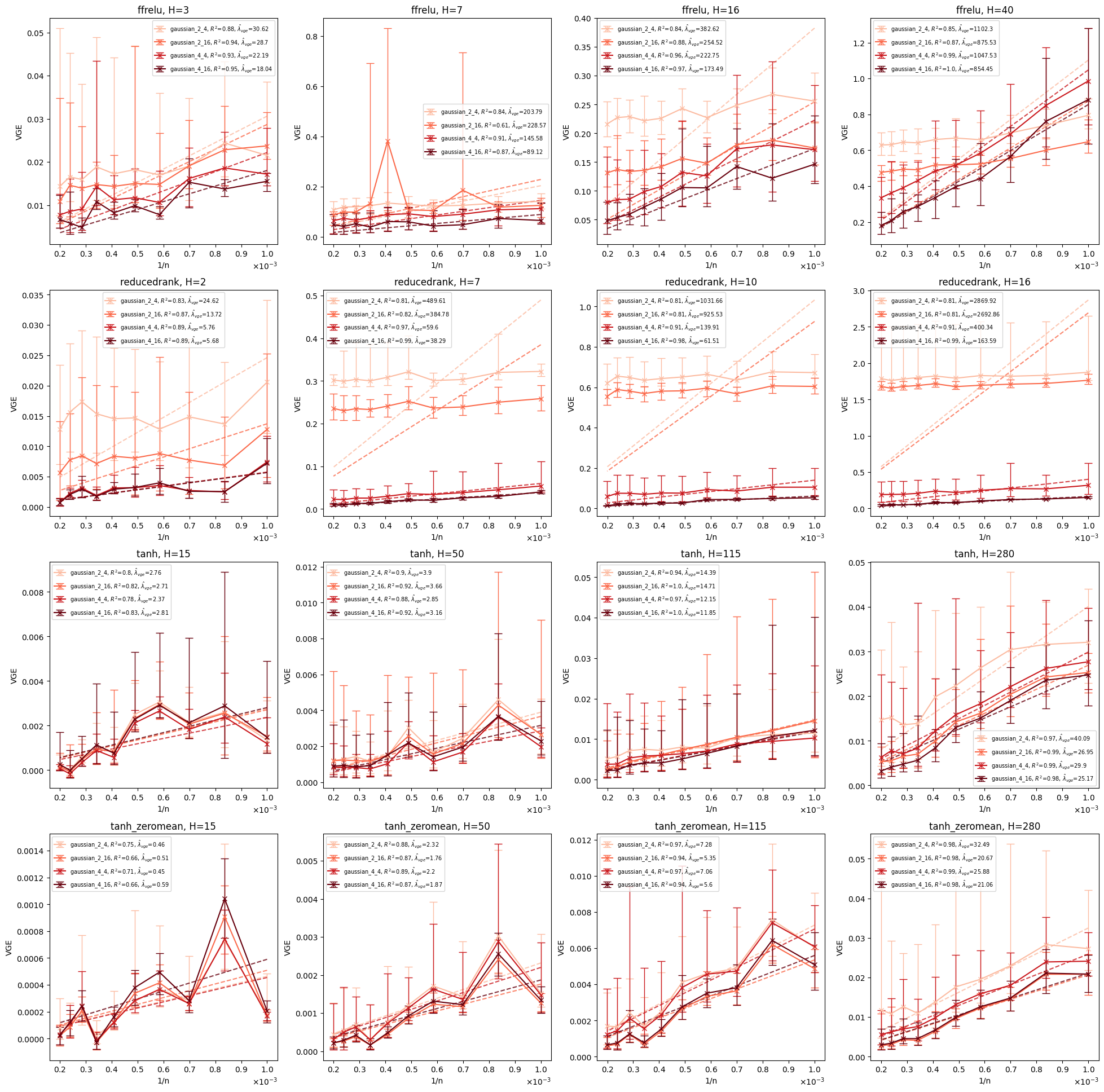}}
\caption{VGE for Gaussian base distribution.}
\label{VGE_vs_inverse_n_gaussian}
\end{center}
\vskip -0.2in
\end{figure}

\begin{figure}[ht]
\vskip 0.2in
\begin{center}
\centerline{\includegraphics[width=\columnwidth]{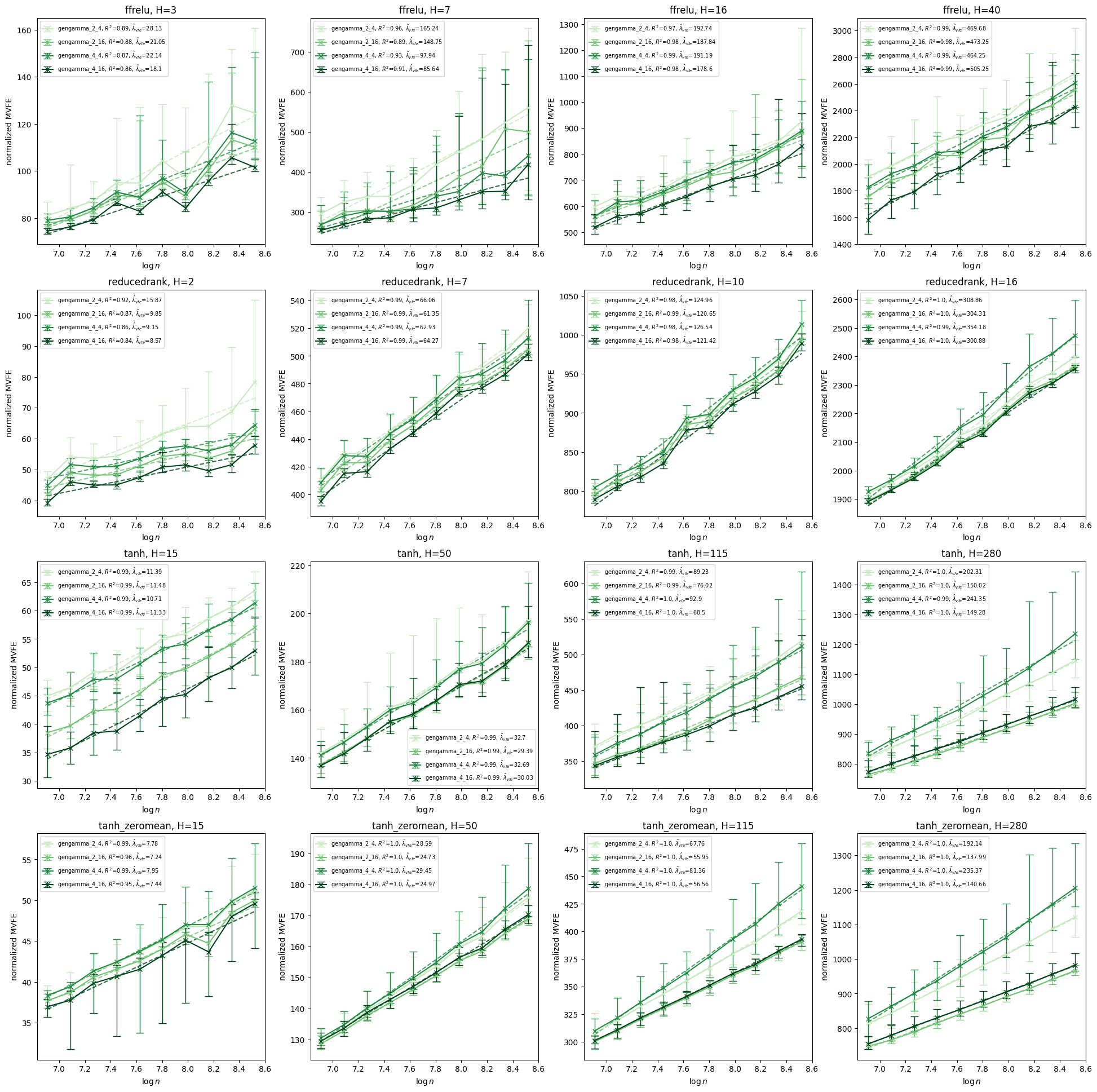}}
\caption{MVFE for generalized gamma base distribution.}
\label{MVFE_vs_logn_gengamma}
\end{center}
\vskip -0.2in
\end{figure}

\begin{figure}[ht]
\vskip 0.2in
\begin{center}
\centerline{\includegraphics[width=\columnwidth]{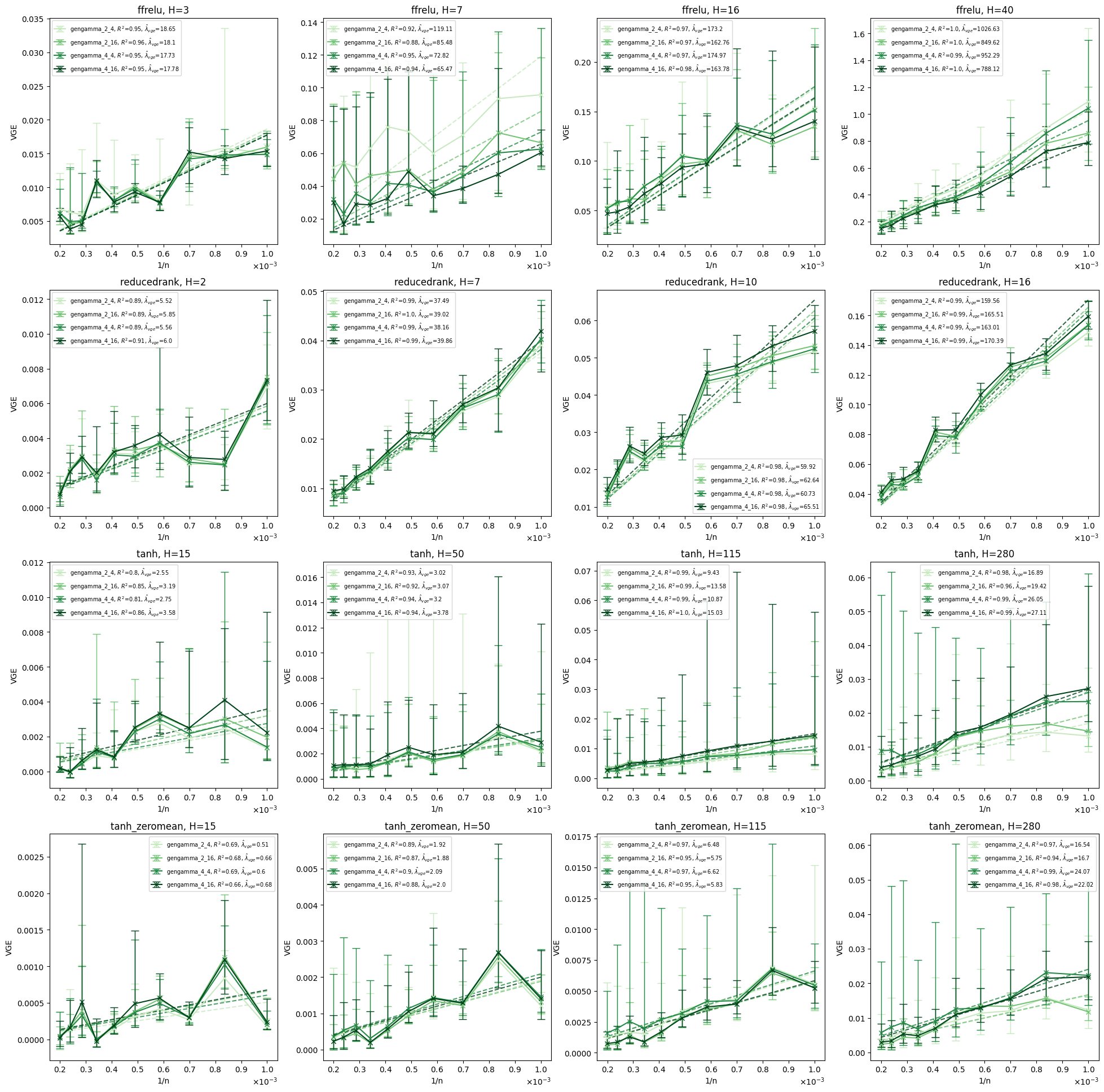}}
\caption{VGE for generalized gamma base distribution.}
\label{VGE_vs_inverse_n_gengamma}
\end{center}
\vskip -0.2in
\end{figure}

\begin{figure}[ht]
\vskip 0.2in
\begin{center}
\centerline{\includegraphics[width=\columnwidth]{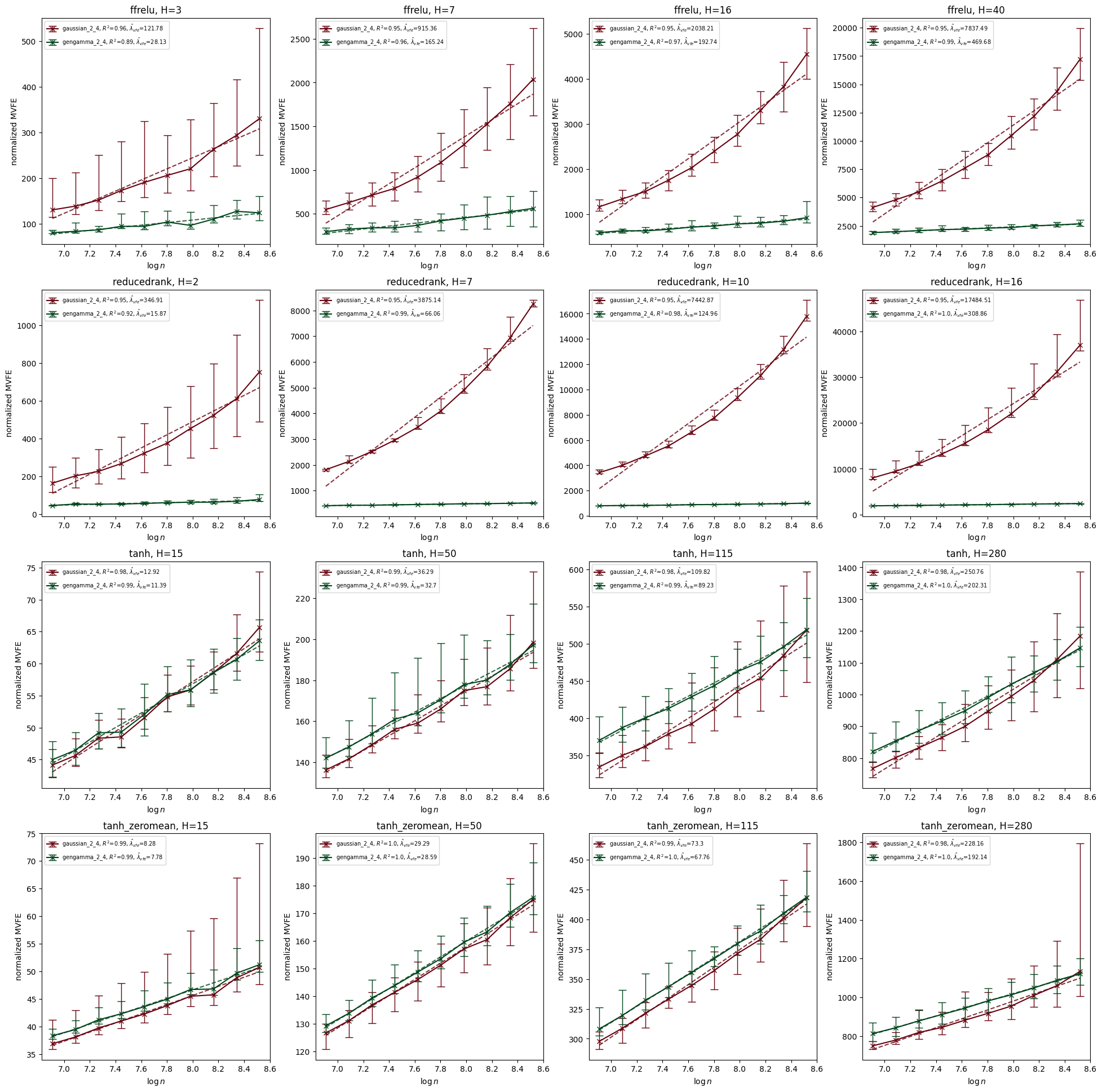}}
\caption{MVFE for $G_\theta$ with the least expressive \texttt{2\_4} configuration.}
\label{MVFE_vs_logn_gengamma_gaussian_2_4_only}
\end{center}
\vskip -0.2in
\end{figure}

\begin{figure}[ht]
\vskip 0.2in
\begin{center}
\centerline{\includegraphics[width=\columnwidth]{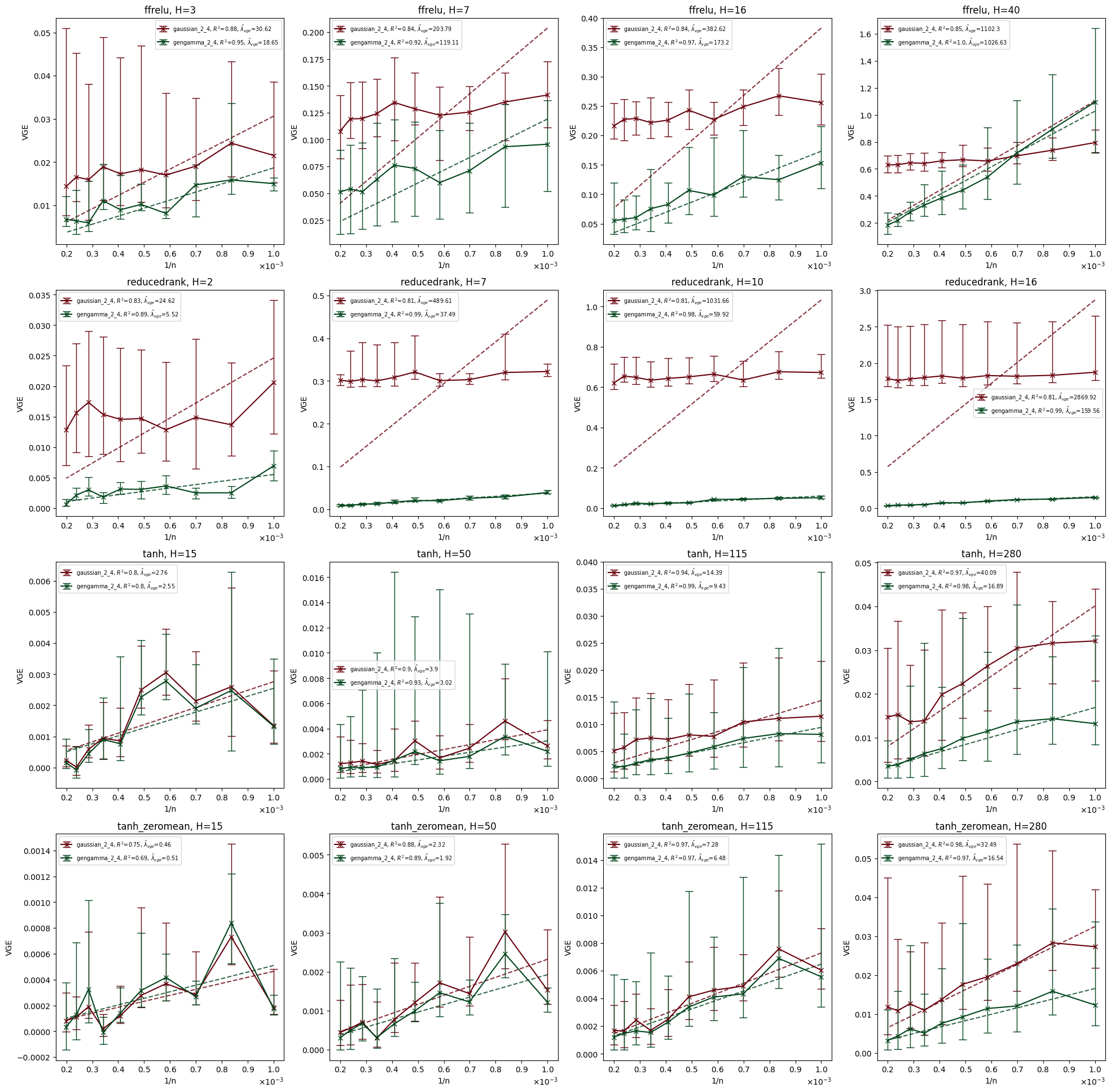}}
\caption{VGE for $G_\theta$ with the least expressive \texttt{2\_4} configuration.}
\label{VGE_vs_inverse_n_gengamma_gaussian_2_4_only}
\end{center}
\vskip -0.2in
\end{figure}

\begin{figure}[ht]
\vskip 0.2in
\begin{center}
\centerline{\includegraphics[width=\columnwidth]{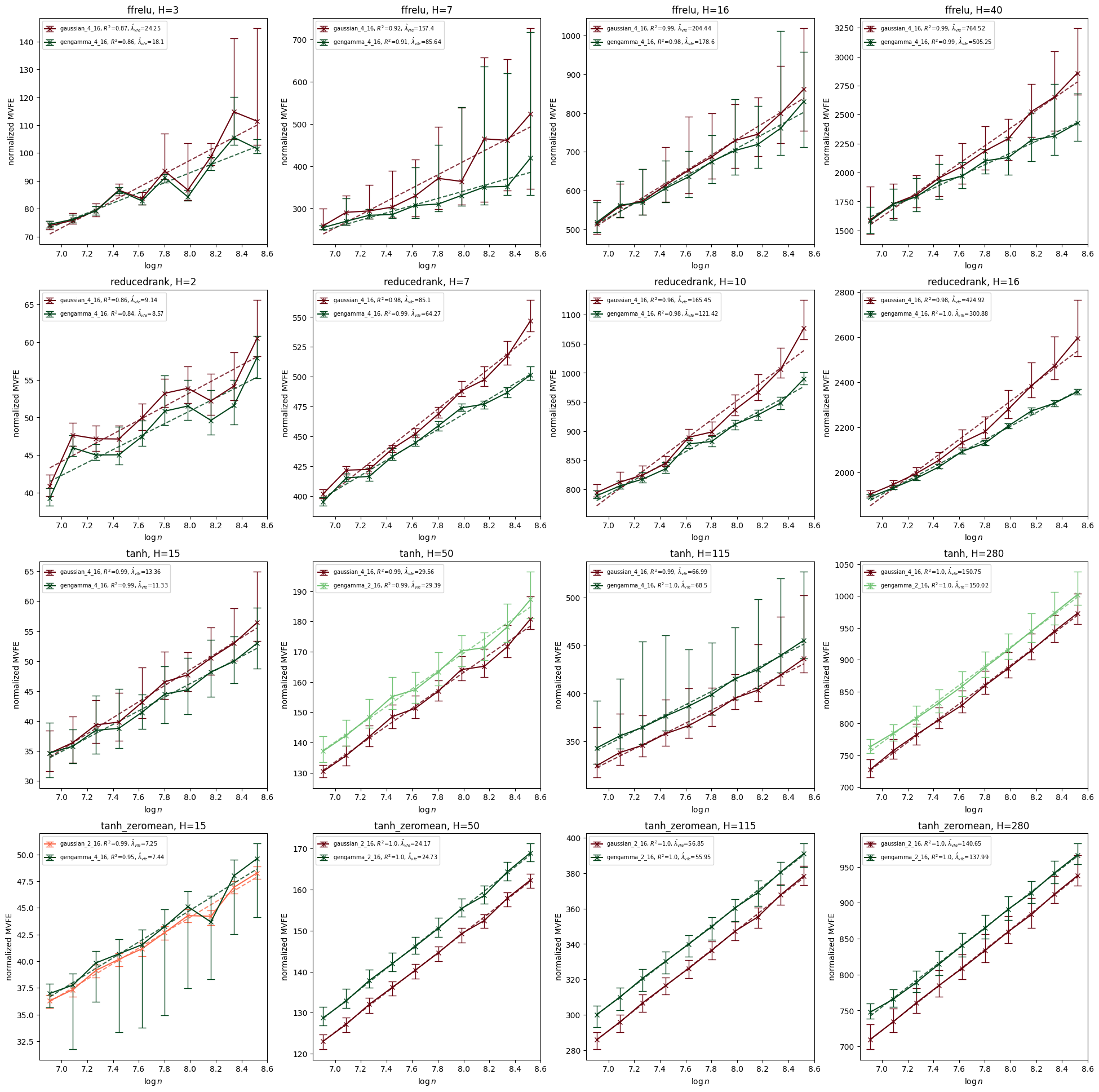}}
\caption{MVFE for $G_\theta$ with the best performing architecture for each base distribution, as judged by MVFE. This is usually the \texttt{4\_16} configuration, but not always.}
\label{MVFE_vs_logn_gengamma_gaussian_best_line_only}
\end{center}
\vskip -0.2in
\end{figure}

\begin{figure}[ht]
\vskip 0.2in
\begin{center}
\centerline{\includegraphics[width=\columnwidth]{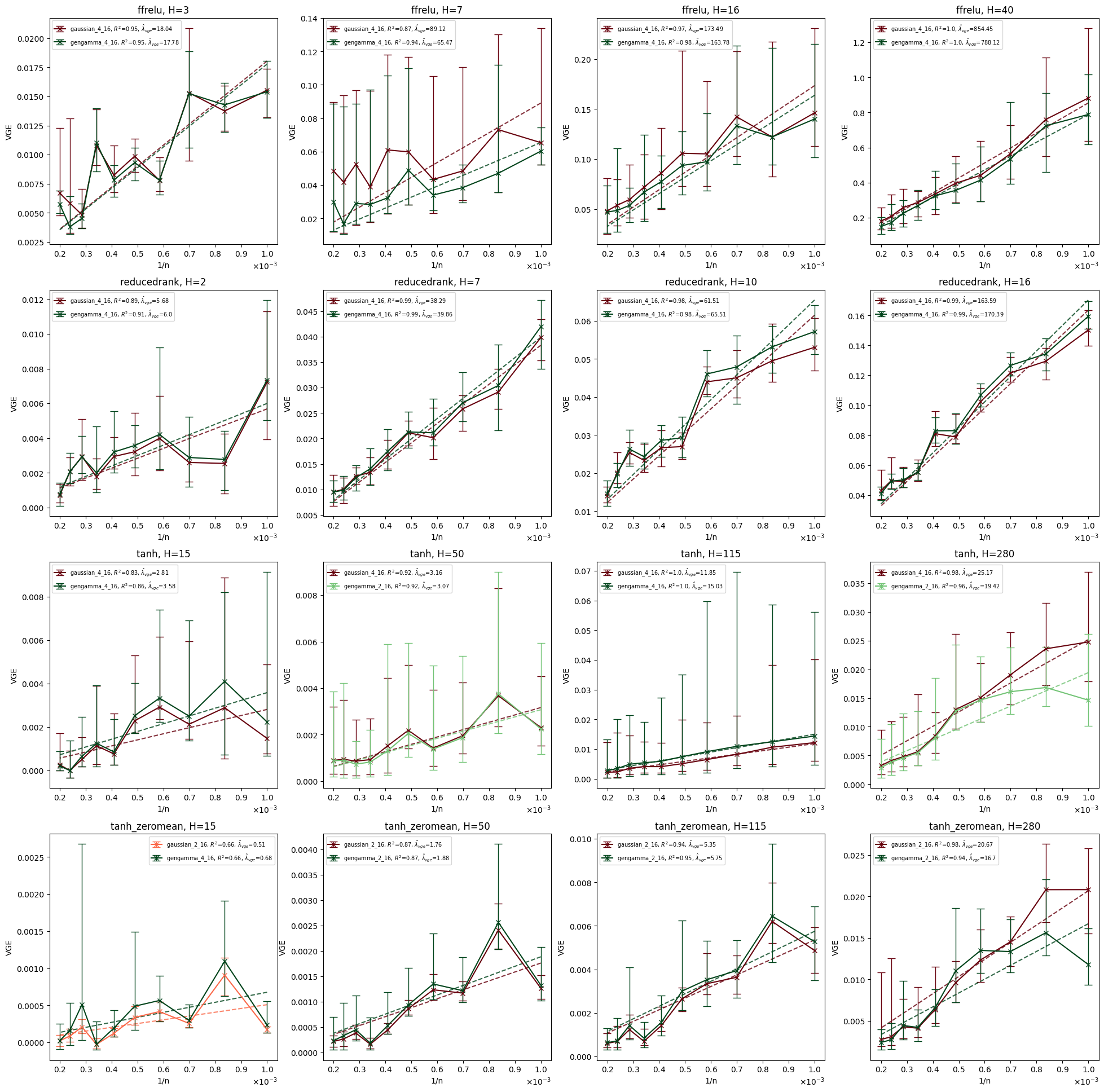}}
\caption{VGE for $G_\theta$ with the best performing architecture for each base distribution, as judged by MVFE. This is usually the \texttt{4\_16} configuration, but not always.}
\label{VGE_vs_inverse_n_gengamma_gaussian_best_line_only}
\end{center}
\vskip -0.2in
\end{figure}

\begin{figure}[ht]
\vskip 0.2in
\begin{center}
\centerline{\includegraphics[width=\columnwidth]{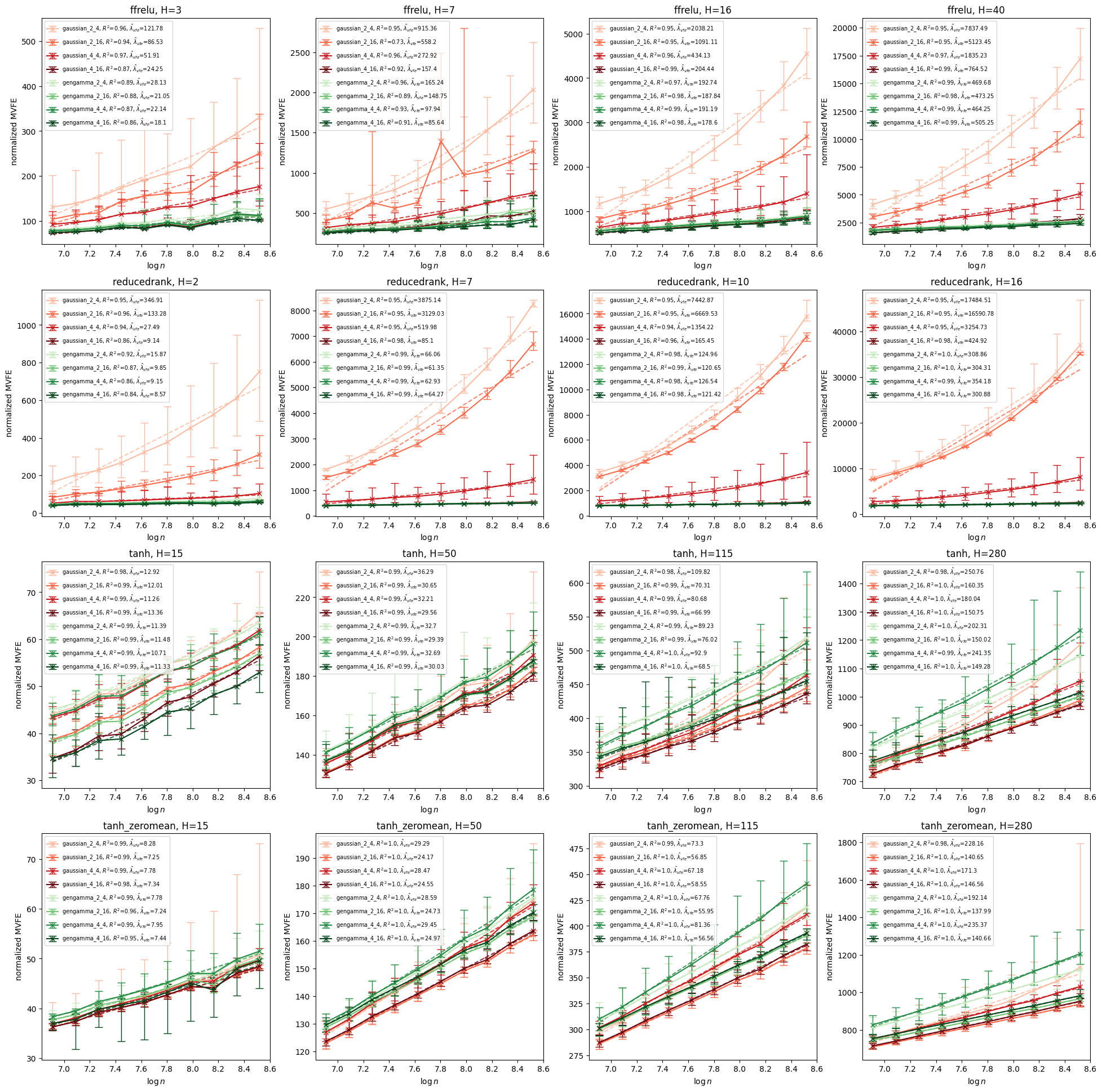}}
\caption{MVFE for all base distributions and $G_\theta$ architectures considered. Note that the first column of Figure \ref{fig:main_text_results} in the main text is a subset of the plots here.}
\label{MVFE_vs_logn_gengamma_gaussian}
\end{center}
\vskip -0.2in
\end{figure}

\begin{figure}[ht]
\vskip 0.2in
\begin{center}
\centerline{\includegraphics[width=\columnwidth]{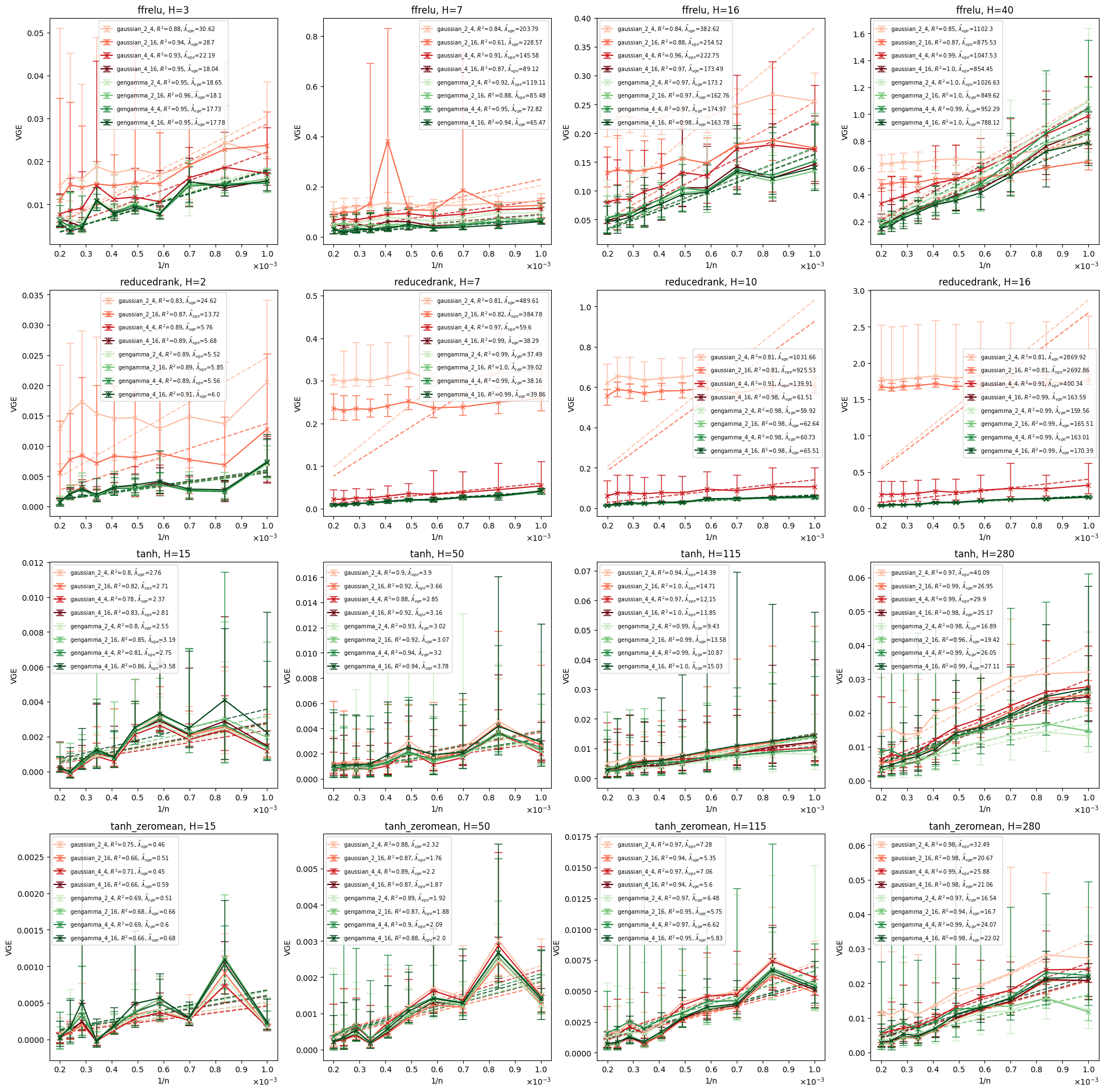}}
\caption{VGE for all base distributions and $G_\theta$ architectures considered. Note that the second column of Figure \ref{fig:main_text_results} in the main text is a subset of the plots here.}
\label{VGE_vs_inverse_n_gengamma_gaussian}
\end{center}
\vskip -0.2in
\end{figure}

\end{document}